\documentclass{article}
\usepackage{amsfonts}
\usepackage{amsmath}
\usepackage{amssymb}
\usepackage[ruled,vlined]{algorithm2e}
\usepackage{bm}
\usepackage{tikz}
\usetikzlibrary{arrows,automata,shapes,decorations,calc,matrix,decorations.pathmorphing}
\usepackage{mathrsfs}
\usepackage{amsthm}

\theoremstyle{plain}
\newtheorem{theorem}{Theorem}
\newtheorem{lemma}[theorem]{Lemma}

\newtheorem{claim}[theorem]{Claim}

\theoremstyle{definition}

\theoremstyle{remark}
\newtheorem{remark}{Remark}

\makeatletter
\newtheorem*{rep@theorem}{\rep@title}
\newcommand{\newreptheorem}[2]{%
\newenvironment{rep#1}[1]{%
 \def\rep@title{#2 \ref{##1}}%
 \begin{rep@theorem}}%
 {\end{rep@theorem}}}
\makeatother

\makeatletter
\newtheorem*{rep@lemma}{\rep@title}
\newcommand{\newreplemma}[2]{%
\newenvironment{rep#1}[1]{%
 \def\rep@title{#2 \ref{##1}}%
 \begin{rep@lemma}}%
 {\end{rep@lemma}}}
\makeatother

\newreplemma{lemma}{Lemma}
\newreptheorem{theorem}{Theorem}

\DeclareMathOperator{\Poly}{Poly}

\newcommand{\algo}[1]{\mathsf{Algo#1}}
\newcommand{\f}{f}
\newcommand{\E}{\mathbb{E}}
\newcommand{\ift}{\text{if }}

\title{Faster Monte-Carlo Algorithms for Fixation Probability of the Moran Process on Undirected Graphs}

\author{Krishnendu Chatterjee, Rasmus Ibsen-Jensen and Martin A. Nowak}

\begin{document}

\maketitle

\begin{abstract}
Evolutionary graph theory studies the evolutionary dynamics in a population 
structure given as a connected graph.
Each node of the graph represents an individual of the population, and edges
determine how offspring are placed.
We consider the classical birth-death Moran process where there are two types
of individuals, namely, the residents with fitness $1$ and mutants with
fitness $r$.
The fitness indicates the reproductive strength.
The evolutionary dynamics happens as follows: in the initial step, 
in a population of all resident individuals a mutant is introduced,
and then at each step, an individual is chosen proportional to the fitness
of its type to reproduce, and the offspring replaces a neighbor uniformly
at random.
The process stops when all individuals are either residents or mutants. 
The probability that all individuals in the end are mutants is called
the fixation probability, which is a key factor in the rate of evolution.
We consider the problem of approximating the fixation probability.

The class of algorithms that is extremely relevant for approximation of 
the fixation probabilities is the Monte-Carlo simulation of the process.
Previous results present a polynomial-time Monte-Carlo algorithm for undirected 
graphs when $r$ is given in unary.
First, we present a simple modification: instead of simulating each step, 
we discard {\em ineffective} steps, where no node changes type 
(i.e., either residents replace residents, or mutants replace mutants).
Using the above simple modification and our result that the number of 
effective steps is concentrated around the expected number of effective steps, 
we present faster polynomial-time Monte-Carlo algorithms for undirected graphs.
Our algorithms are always at least a factor $O(n^2/\log n)$ faster as 
compared to the previous algorithms, where $n$ is the number of nodes, 
and is polynomial even if $r$ is given in binary.
We also present lower bounds showing that the upper bound on the expected 
number of effective steps we present is asymptotically tight for undirected 
graphs.
\end{abstract}

\section{Introduction}\label{sec:intro}
In this work we present faster Monte-Carlo algorithms for approximation 
of the fixation probability of the fundamental Moran process on population 
structures with symmetric interactions.
We start with the description of the problem.

\subparagraph*{Evolutionary dynamics} 
Evolutionary dynamics act on populations, where the composition of the 
population changes over time due to mutation and selection. Mutation generates 
new types and selection changes the relative abundance of different types. 
A fundamental concept in evolutionary dynamics is the fixation probability of 
a new mutant~\cite{Ewens04,Karlin75,Moran62,NowakBook}: Consider a population 
of $n$ \emph{resident} individuals, each with a fitness 
value $1$. 
A single \emph{mutant} with non-negative fitness value~$r$ is introduced in 
the population as the initialization step. 
Intuitively, the fitness represents the reproductive strength. 
In the classical Moran process the following {\em birth-death} stochastic 
steps are repeated:
At each time step, one individual is chosen at random proportional to the fitness 
to reproduce and one other individual is chosen uniformly at random for death. 
The offspring of the reproduced individual replaces the dead individual. 
This stochastic process continues until either all individuals are mutants or 
all individuals are residents. 
The \emph{fixation probability} is the probability that the mutants take 
over the population, which means all individuals are mutants.
A standard calculation shows that the fixation probability is given by 
$(1-(1/r))/(1-(1/r^n))$.
The correlation between the relative fitness $r$ of the mutant 
and the fixation probability is a measure of the effect of natural selection.
The rate of evolution, which is the rate at which subsequent mutations 
accumulate in the population, 
is proportional to the fixation probability, the mutation rate, and 
the population size $n$. 
Hence fixation probability is a fundamental concept in evolution.

\subparagraph*{Evolutionary graph theory}
While the basic Moran process happens in well-mixed population (all
individuals interact uniformly with all others), a fundamental 
extension is to study the process on population structures.
Evolutionary graph theory studies this phenomenon. 
The individuals of the population occupy the nodes of a connected graph. 
The links (edges) determine who interacts with whom. 
Basically, in the birth-death step, for the death for replacement, 
a neighbor of the reproducing individual is chosen uniformly at random.
Evolutionary graph theory describes evolutionary dynamics in spatially 
structured population where most interactions and competitions occur 
mainly among neighbors in physical space~\cite{Nowak05,Hauert14,Frean07,Shakarian12}. 
Undirected graphs represent population structures where the interactions
are symmetric, whereas directed graphs allow for asymmetric interactions.
The fixation probability depends on the population structure~\cite{Nowak05,ACN15,ICALP16,Daz13}. 
Thus, the fundamental computational problem in evolutionary graph theory 
is as follows: 
given a population structure (i.e., a graph), the relative fitness $r$,
and $\epsilon>0$, compute an $\epsilon$-approximation of the fixation 
probability.

\subparagraph*{Monte-Carlo algorithms}
A particularly important class of algorithms for biologists is the 
Monte-Carlo algorithms, because it is simple and easy to interpret. 
The Monte-Carlo algorithm for the Moran process basically requires to 
simulate the process, and from the statistics obtain an approximation of 
the fixation probability.
Hence, the basic question we address in this work is simple Monte-Carlo
algorithms for approximating the fixation probability. 
It was shown in \cite{Diaz16} that simple simulation can take exponential time on directed graphs and thus we focus on undirected graphs.
The main previous algorithmic result in this area~\cite{Diaz14} presents 
a polynomial-time Monte-Carlo algorithm for undirected graphs when $r$ is 
given in unary.
The main result of~\cite{Diaz14} shows that for undirected graphs it suffices to 
run each simulation for polynomially many steps.

\subparagraph*{Our contributions} 
In this work our main contributions are as follows:

\begin{enumerate}
\item {\em Faster algorithm for undirected graphs}
First, we present a simple modification: instead of simulating each step, 
we discard {\em ineffective} steps, where no node changes type 
(i.e., either residents replace residents, or mutants replace mutants).
We then show that the number of effective steps is concentrated around 
the expected number of effective steps.
The sampling of each effective step is more complicated though than sampling of
each step.
We then present an efficient algorithm for sampling of the effective steps,
which requires $O(m)$ preprocessing and then $O(\Delta)$ time for sampling,
where $m$ is the number of edges and $\Delta$ is the maximum degree.
Combining all our results we obtain faster polynomial-time Monte-Carlo 
algorithms:
Our algorithms are always at least a factor $n^2/\log n$ times a constant
(in most cases $n^3/\log n$ times a constant) faster as compared 
to the previous algorithm, and is polynomial even if $r$ is given in binary.
We present a comparison in Table~\ref{tab:intro}, for constant $r>1$ 
(since the previous algorithm is not in polynomial time for $r$ in binary).
For a detailed comparison see Table~\ref{tab:appendix} in the Appendix.
 
\begin{table}
\center
\begin{tabular}{| l l| c| c| }
\hline
   & & All steps & Effective steps \\\hline
    \#steps in expectation& & $O(n^2\Delta^2)$ & $O(n\Delta)$ \\\hline
    Concentration bounds& & $\Pr[\tau \geq \frac{n^2\Delta^2rx}{r-1}]\leq 1/x$ & $\Pr[\tau \geq \frac{6n\Delta x}{\min(r-1,1)}]\leq 2^{-x}$ \\
    \hline
\multicolumn{2}{|l|}{Sampling a step}  & $O(1)$ & $O(\Delta)$ \\\hline
\multicolumn{2}{|l|}{Fixation algo}  & $O(n^6 \Delta^2\epsilon^{-4})$ & $O(n^2 \Delta^2\epsilon^{-2}(\log n+\log\epsilon^{-1}))$ \\\hline
\end{tabular}
\caption{Comparison with previous work, for constant $r>1$. 
We denote by $n$, $\Delta$, $\tau$, and $\epsilon$, the number of nodes, the maximum degree, the random variable for the fixation time, and
the approximation factor, respectively.
The results in the column ``All steps'' is from \cite{Diaz14}, except that we present the dependency on $\Delta$, which was considered as $n$ in~\cite{Diaz14}. 
The results of the column ``Effective steps'' is the results of this paper\label{tab:intro}}
\end{table}

\item {\em Lower bounds}
We also present lower bounds showing that the upper bound on the expected 
number of effective steps we present is asymptotically tight for undirected graphs.

\end{enumerate}

\subparagraph*{Related complexity result}
While in this work we consider evolutionary graph theory, a related problem 
is evolutionary games on graphs (which studies the problem of frequency 
dependent selection). 
The approximation problem for evolutionary games on graphs is considerably
harder (e.g., PSPACE-completeness results have been established)~\cite{ICN15}.

\subparagraph*{Technical contributions}
Note that for the problem we consider the goal is not to design complicated
efficient algorithms, but simple algorithms that are efficient. 
By simple, we mean something that is related to the process itself,
as biologists understand and interpret the Moran process well.
Our main technical contribution is a simple idea to discard ineffective steps,
which is intuitive, and we show that the simple modification leads to 
significantly faster algorithms.
We show a gain of factor $O(n\Delta)$ due to the effective steps, then lose
a factor of $O(\Delta)$ due to sampling, and our other improvements are due to 
better concentration results.
We also present an interesting family of graphs for the lower bound examples.
Technical proofs omitted due to lack of space are in the Appendix.

\section{Moran process on graphs}\label{sec:def}

\subparagraph*{Connected graph and type function}
We consider the population structure represented as a connected graph.
There is a connected {\em graph} $G=(V,E)$, of $n$ nodes and $m$ edges, 
and two types $T=\{t_1,t_2\}$. 
The two types represent residents and mutants, and in the technical exposition 
we refer to them as $t_1$ and $t_2$ for elegant notation.
We say that a node $v$ is a {\em successor} of a node $u$ if $(u,v)\in E$.
The graph is {\em undirected} if for all $(u,v) \in E$ we also have $(v,u) \in E$, 
otherwise it is {\em directed}.
There is a {\em type function} $\f$ mapping each node $v$ to a type $t\in T$. 
Each type $t$ is in turn associated with a positive integer $w(t)$, the type's fitness
denoting the corresponding reproductive strength. 
Without loss of generality, we will assume that $r=w(t_1)\geq w(t_2)=1$, for some number 
$r$ 
(
the process we consider does not change under scaling, and $r$ denotes relative fitness). 
Let $W(\f)=\sum_{u\in V} w(\f(u))$ be the total fitness. 
For a node $v$ let $\deg v$ be the degree of $v$ in $G$. 
Also, let $\Delta=\max_{v\in V} \deg v$ be the maximum degree of a node. 
For a type $t$ and type function $\f$, let $V_{t,\f}$ be the nodes mapped to $t$ by $\f$.
Given a type $t$ and a node $v$, let $\f[v\rightarrow t]$ denote the following 
function: 
$f[v\rightarrow t](u)= t$ $\ift u=v$ and $\f(u)$ otherwise.

\subparagraph*{Moran process on graphs}
We consider the following classical Moran birth-death process where a 
{\em dynamic evolution step} of the process changes a type function from $\f$ to $\f'$
as follows: 
\begin{enumerate}

\item First a node $v$ is picked at random with probability proportional to $w(\f(v))$, i.e. each node $v$ has probability of being picked equal to $\frac{w(\f(v))}{W(\f)}$. 
\item Next, a successor $u$ of $v$ is picked uniformly at random.
\item The type of $u$ is then changed to $\f(v)$. In other words, 
$\f'=\f[u\rightarrow \f(v)]$.

\end{enumerate}

\subparagraph*{Fixation} 
A type $t$ {\em fixates} in a type function $\f$ if $\f$ maps all nodes to $t$.
Given a type function $\f$, repeated applications of the dynamic evolution step 
generate a sequence of type functions $\f=\f_1,\f_2,\dots,\f_\infty$. 
Note that if a type has fixated (for some type $t$) in $\f_i$ then it has also fixated in $\f_j$ for $i<j$. We say that a process has {\em fixation time} $i$ if $\f_i$ has fixated but $\f_{i-1}$ has not. We say that an initial type function $\f$ has fixation probability $p$ for a type $t$, if the probability that $t$ eventually fixates (over the probability measure on sequences generated by repeated applications of the dynamic evolution step $\f$)

\subparagraph*{Basic questions}
We consider the following basic questions:
\begin{enumerate}

\item {\em Fixation problem} Given a type $t$, what is the fixation probability of $t$ averaged over the $n$ initial type functions with a single node mapping to $t$?
\item {\em Extinction problem} Given a type $t$, what is the fixation probability of $t$ averaged over the $n$ initial type functions with a single node {\em not} mapping to $t$?
\item {\em Generalized fixation problem} Given a graph, a type $t$ and an type function $\f$ what is the fixation probability of $t$ in $G$, when the initial type function is $\f$? 
\end{enumerate}

\begin{remark}\label{rem:neutral}
Note that in the {\em neutral} case when $r=1$, the fixation problem has answer $1/n$ and 
extinction problem has answer $1-1/n$. Hence, in the rest of the paper we will consider
$r >1$.
Also, to keep the presentation focused, in the main article, we will consider fixation 
and extinction of type $t_1$.
In the Appendix we also present another algorithm for the extinction of $t_2$. 
\end{remark}

\subparagraph*{Results}
We will focus on undirected graphs.
For undirected graphs, we will give new FPRAS (fully polynomial, randomized 
approximation scheme) for the fixation and the extinction problem, 
and a polynomial-time algorithm for an additive approximation of the 
generalized fixation problem. 
There exists previous FPRAS for the fixation and extinction problems~\cite{Diaz14}. 
Our upper bounds are at least a factor of $O(\frac{n^2}{\log n})$ (most cases $O(\frac{n^3}{\log n})$)
better and always in $\Poly(n,1/\epsilon)$, whereas the previous algorithms are not in polynomial 
time for $r$ given in binary.

\section{Discarding ineffective steps}\label{sec:ineffective}
We consider undirected graphs. 
Previous work by Diaz et al.~\cite{Diaz14} showed that the expected number of dynamic evolution 
steps till fixation is polynomial, and then used it to give a polynomial-time 
Monte-Carlo algorithm. 
Our goal is to improve the quite high polynomial-time complexity, 
while giving a Monte-Carlo algorithm. 
To achieve this we define the notion of effective steps.

\subparagraph*{Effective steps}
A dynamic evolution step, which changes the type function from $\f$ to $\f'$, 
is {\em effective} if $\f\neq \f'$ (and {\em ineffective} otherwise).
The idea is that steps in which no node changes type (because the two nodes 
selected in the dynamic evolution step already had the same type) 
can be discarded, without changing which type fixates/gets eliminated. 

\subparagraph*{Two challenges} The two challenges are as follows:
\begin{enumerate}
\item {\em Number of steps} 
The first challenge is to establish that the expected number of effective
steps is asymptotically smaller than the expected number of all steps.
We will establish a factor $O(n\Delta)$ improvement (recall $\Delta$ is the 
maximum degree).
\item {\em Sampling} 
Sampling an effective step is harder than sampling a normal step. 
Thus it is not clear that considering effective steps leads to a faster algorithm.
We consider the problem of efficiently sampling an effective step in a later section, see Section~\ref{sec:sample}. We show that sampling an effective step can be done in $O(\Delta)$ time (after $O(m)$ preprocessing).
\end{enumerate}

\subparagraph*{Notation} 
For a type function $\f$, let $\Gamma_v(\f)$ be the subset of successors of $v$, such that $u\in \Gamma_v(\f)$ iff  
$\f(v)\neq \f(u)$. Also, let $W'(\f)=\sum_u w(\f(u))\cdot \frac{|\Gamma_u(\f)|}{\deg u}$.

\subparagraph*{Modified dynamic evolution step}
Formally, we consider the following {\em modified dynamic evolution step} (that changes the type function from $\f$ to $\f'$ and assumes that $f$ does not map all nodes to the same type): 
\begin{enumerate}
\item First a node $v$ is picked at random  with probability proportional to $p(v)=w(\f(v))\cdot \frac{|\Gamma_v(\f)|}{\deg v}$ i.e. each node $v$ has probability of being picked equal to $\frac{p(v)}{W'(\f)}$.
\item Next, a successor $u$ of $v$ is picked uniformly at random among $\Gamma_v(\f)$.
\item The type of $u$ is then changed to $\f(v)$, i.e., $\f'=\f[u\rightarrow \f(v)]$.
\end{enumerate}

In the following lemma we show that the modified dynamic evolution step corresponds to the dynamic evolution 
step except for discarding steps in which no change was made.

\begin{lemma}\label{lem:preserve}
Fix any type function $\f$ such that neither type has fixated. 
Let $\f_d$ (resp., $\f_m$) be the next type function under dynamic evolution step (resp., modified dynamic evolution step). 
Then, $\Pr[\f\neq \f_d]> 0$ and for all type functions $\f'$ we have: $\Pr[\f'=\f_d\mid \f\neq \f_d]=\Pr[\f'=\f_m]$.
\end{lemma}

\subparagraph*{Potential function $\psi$}
Similar to~\cite{Diaz14} we consider the {\em potential function} $\psi=\sum_{v\in V_{t_1,\f}} \frac{1}{\deg v}$
(recall that $V_{t_1,\f}$ is the set of nodes of type $t_1$). 
We now lower bound the expected difference in potential per modified
evolutionary step.

\begin{lemma}\label{lem:fast1step}
Let $\f$ be a type function such that neither type has fixated.
Apply a modified dynamic evolution step on $\f$ to obtain $\f'$. 
Then, \[
\E[\psi(\f')-\psi(\f)]\geq \frac{r-1}{\Delta\cdot (r+1)} \enspace .
\]
\end{lemma}
\begin{proof}
Observe that $\f$ differs from $\f'$ for exactly one node $u$. More precisely, let $v$ be the node picked in line~1 of the modified dynamic evolution step and let $u$ be the node picked in line~2.
Then, $\f'=\f[u\rightarrow \f(v)]$. The probability to select $v$ is $\frac{p(v)}{W'(\f)}$. The probability to then pick $u$ is $\frac{1}{|\Gamma_v(\f)|}$.

We have that
\begin{itemize}
\item If $\f(u)=t_2$ (and thus, since it got picked $\f(v)=t_1$), then $\psi(\f')-\psi(\f)=\frac{1}{\deg u}$.
\item If $\f(u)=t_1$ (and thus, since it got picked $\f(v)=t_2$), then $\psi(\f')-\psi(\f)=-\frac{1}{\deg u}$.
\end{itemize}
Below we use the following notations:
\[
E_{12}=\{(v,u)\in E\mid \f(v)=t_1 \text{ and } \f(u)=t_2\}; \quad
E_{21}=\{(v,u)\in E\mid \f(v)=t_2 \text{ and } \f(u)=t_1\}.
\]
Thus, 
\begin{align*}
\E[\psi(\f')-\psi(\f)]&= 
\sum_{(v,u)\in E_{12} }\left(\frac{p(v)}{W'(\f)}\cdot \frac{1}{|\Gamma_v(\f)|}\cdot \frac{1}{\deg u}\right)-
\sum_{(v,u)\in E_{21}}\left(\frac{p(v)}{W'(\f)}\cdot \frac{1}{|\Gamma_v(\f)|}\cdot \frac{1}{\deg u}\right) \\
&= \sum_{(v,u)\in E_{12}}\left(\frac{w(\f(v))}{W'(\f)\cdot (\deg u)\cdot (\deg v)}\right)-
\sum_{(v,u)\in E_{21}}\left(\frac{w(\f(v))}{W'(\f)\cdot (\deg u)\cdot (\deg v)}\right) \enspace .
\end{align*}

Using that the graph is undirected we get,
\begin{align*}
\E[\psi(\f)-\psi(\f')]&= \sum_{(v,u)\in E_{12}}\left(\frac{w(\f(v))-w(\f(u))}{W'(\f)\cdot (\deg u)\cdot (\deg v)}\right)\\
&= \frac{1}{W'(\f)}\sum_{(v,u)\in E_{12}}\left(\frac{r-1}{\min(\deg u,\deg v)\cdot \max(\deg u,\deg v)}\right)\\
&\geq 
\frac{r-1}{\Delta\cdot W'(\f)}\sum_{(v,u)\in E_{12}}\frac{1}{\min(\deg u,\deg v)}
=\frac{r-1}{\Delta\cdot W'(\f)}\cdot S
\enspace ,
\end{align*}
where $S=\sum_{(v,u)\in E_{12}}\frac{1}{\min(\deg u,\deg v)}$. 
Note that in the second equality we use that for two numbers $a,b$, their product
is equal to $\min(a,b) \cdot \max(a,b)$.
By definition of $W'(\f)$, we have \begin{align*}
W'(\f)&=\sum_u w(\f(u))\cdot \frac{|\Gamma_u(\f)|}{\deg u}=\sum_u \sum_{v\in \Gamma_u(\f)}\frac{w(\f(u))}{\deg u}=
\sum_{\substack{(v,u)\in E\\\f(u)\neq f(v)}} \frac{w(\f(u))}{\deg u}\\
&=\sum_{(v,u)\in E_{12}} \left(\frac{w(\f(u))}{\deg u}+\frac{w(\f(v))}{\deg v}\right)\leq 
\sum_{(v,u)\in E_{12}} \frac{w(\f(u))+w(\f(v))}{\min(\deg u,\deg v)} = (r+1)\cdot S\enspace .
\end{align*}
Thus, we see that 
$\E[\psi(\f')-\psi(\f)]\geq \frac{r-1}{\Delta\cdot (r+1)}$,
as desired. This completes the proof.
\end{proof}

\begin{lemma}\label{lem:large_r}
Let $r=x\Delta$ for some number $x>0$.
Let $\f$ be a type function such that neither type has fixated.
Apply a modified dynamic evolution step on $\f$ to obtain $\f'$. The probability that 
$|V_{t_1,\f'}|=|V_{t_1,\f}|+1$ is at least $\frac{x}{x+1}$ 
(otherwise, $|V_{t_1,\f'}|=|V_{t_1,\f}|-1$). 
\end{lemma}
\begin{proof}
Consider any type function $\f$. Let $m'$ be the number of edges $(u,v)$, such that $\f(u)\neq \f(v)$. We will argue that the total weight of nodes of type $t_1$, denoted $W_1$, is at least $xm'$ and that the total weight of nodes of type $t_2$, denoted $W_2$, is at most $m'$. We see that as follows: 
\[
W_1=\sum_{v\in V_{t_1,\f}}w(\f(v))\cdot \frac{|\Gamma_v(\f)|}{\deg v}\geq \frac{x\Delta}{\Delta}\sum_{v\in V_{t_1,\f}}|\Gamma_v(\f)|=xm'\enspace ,
\]
using that $\deg v\leq \Delta$ and $w(\f(v))=r=x\Delta$ in the inequality.
Also,
\[
W_2=\sum_{v\in V_{t_2,\f}}w(\f(v))\cdot \frac{|\Gamma_v(\f)|}{\deg v}\leq \sum_{v\in V_{t_2,\f}}|\Gamma_v(\f)|=m'\enspace ,
\]
using that $\deg v\geq 1$ and $w(\f(v))=1$ in the inequality.
We see that we thus have a probability of at least $\frac{x}{x+1}$ to pick a node of type $t_1$. Because we are using effective steps, picking a member of type $t$ will increment the number of that type (and decrement the number of the other type).
\end{proof}

\begin{lemma}\label{lemm:conc}
Consider an upper bound $\ell$, for each starting type function, on the expected number of (effective) 
steps to fixation.
Then for any starting type function the probability that fixation requires more than 
$2\cdot \ell\cdot x$ (effective) steps is at most $2^{-x}$.
\end{lemma}
\begin{proof}
By Markov's inequality after $2\cdot \ell$  (effective) steps 
the Moran process fixates with probability at least $\frac{1}{2}$, irrespective of the initial type function. 
We now split the steps into blocks of length $2\cdot \ell$. 
In every block, by the preceding argument, there is a probability of at least $\frac{1}{2}$ to fixate in 
some step of that block, given that the process has not fixated before that block. 
Thus, for any integer $x\geq 1$, the probability to  not fixate before the end of block $x$, 
which happens at step $2\cdot \ell\cdot x$ 
is at most $2^{-x}$. 
\end{proof}

We now present the main theorem of this section, which we obtain using the above lemmas, and 
techniques from~\cite{Diaz14}.

\begin{theorem}\label{thm:fixationtime}
Let $t_1$ and $t_2$ be the two types, such that $r=w(t_1) >  w(t_2)=1$. 
Let $\Delta$ be the maximum degree.  Let $k$ be the number of nodes of type $t_2$ in the initial type function. 
The following assertions hold:
\begin{itemize}
\item {\em Bounds dependent on $r$}
\begin{enumerate}
\item {\em Expected steps} The process requires at most $3k\Delta/\min(r-1,1)$ effective steps in expectation, before fixation is reached.
\item {\em Probability} For any integer $x\geq 1$, after $6xn\Delta/\min(r-1,1)$ effective steps, the probability that the process has not 
fixated is at most $2^{-x}$, irrespective of the initial type function.
\end{enumerate}

\item {\em Bounds independent on $r$}
\begin{enumerate}
\item  {\em Expected steps}
The process requires at most $2nk\Delta^2$ effective steps in expectation, before fixation is reached.
\item  {\em Probability}
For any integer $x\geq 1$, after $4 x n^2 \Delta^2$ effective steps, the probability that the process has not fixated is at most $2^{-x}$,  irrespective of the initial type function.
\end{enumerate}

\item {\em Bounds for $r\geq 2\Delta$} 
\begin{enumerate}
\item {\em Expected steps} The process requires at most $3k$ effective steps in expectation, before fixation is reached.
\item  {\em Probability} For any integer $x\geq 1$, after $6 x n$ effective steps, the probability that the process has not fixated is at most $2^{-x}$,  irrespective of the initial type function.
\end{enumerate}
\end{itemize}
\end{theorem}

\section{Lower bound for undirected graphs}\label{sec:lbundirected}
In this section, we will argue that our bound on the expected number of effective steps is essentially tight, 
for fixed $r$.

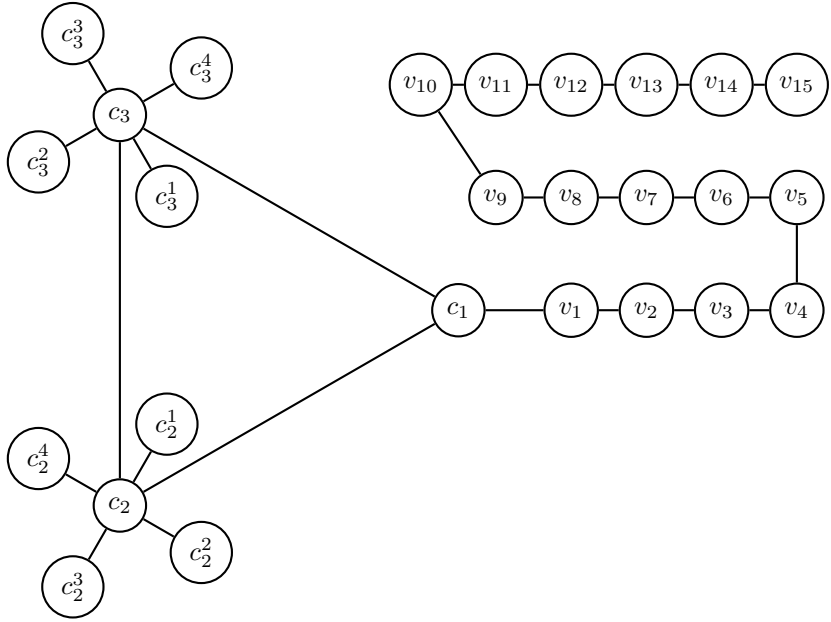
\begin{figure}
\center
\begin{tikzpicture}[thick,scale=0.5]

  \foreach \x/\y/\i in {-120/a/2,-240/b/3}{
	\node[draw,circle] at (\x:6cm) (v\y) {$c_{\i}$};  
	\begin{scope}[shift={(\x:6cm)},rotate=\x]
	\foreach \z/\w/\j in {180/white/1,90/white/2,0/white/3,-90/white/4}{
	\node[draw,circle,fill=\w] at (\z:2.5cm) (v\y\z) {$c_{\i}^{\j}$}; 
	\draw (v\y) -- (v\y\z);
	}
	\end{scope}
  }
	\node[draw,circle] at (0:6cm) (vc) {$c_1$};    
  
  \draw (va) -- (vb) -- (vc)--(va);

  \foreach \x in {1,..., 4}{
	\node[draw,circle] at ($(0:7cm)+(2*\x cm,0)$) (v\x) {$v_{\x}$};  
  }
  
  \foreach \x in {5,..., 9}{
	\node[draw,circle] at ($(v4)+(-2*\x cm,0)+(10cm,3cm)$) (v\x) {$v_{\x}$};  
  }
  
  \foreach \x in {10,..., 15}{
	\node[draw,circle] at ($(v9)+(2*\x cm,0)+(-22cm,3cm)$) (v\x) {$v_{\x}$};  
  }

  \newcommand*{\lastx}{c}
   \foreach \x[remember=\x as \lastx] in {1,..., 15}{
   \draw (v\x) -- (v\lastx);
   }

\end{tikzpicture}
\caption{Example of a member of the family that attains the lower bound for undirected graphs. (Specifically, it is $G^{6,31}$)}\label{fig:undir}
\end{figure}

We construct our lower bound graph $G^{\Delta,n}$, for given $\Delta$, $n$ (sufficiently large), but fixed $r>1$, as follows.
We will argue that fixation of $G^{\Delta,n}$ takes $\Omega(k\Delta)$ effective steps, if there are initially exactly $k$ members of type $t_2$.
For simplicity, we consider  $\Delta>2$ and $n>4\Delta$ (it is easy to see using similar techniques that for lines, where $\Delta=2$, the expected fixation time is $\Omega(k)$ - basically because $t_1$ is going to fixate with pr. $\approx 1-1/r$, using a proof like Lemma~\ref{lem:t_1fix}, and converting the $k$ nodes of type $t_2$ takes at least $k$ efficient steps).
There are two parts to the graph: A line of $\approx n/2$ nodes and a stars-on-a-cycle graph of $\approx n/2$. There is 1 edge from the one of the stars in the stars-on-a-cycle graph to the line.
More formally, the graph is as follows: 
Let $x:=\lfloor{n/(2\Delta-2)}\rfloor$.
There are nodes $V_C=\{c_1,\dots, c_x\}$, such that $c_i$ is connected to $c_{i-1}$ and $c_{i+1}$ for $1<i<x$. Also, $c_1$ is connected to $c_x$. The nodes $V_C$ are the centers of the stars in the stars-on-a-cycle graph. For each $i$, such that $2\leq  i\leq x$, the node $c_i$ is connected to a set of leaves $V_C^{i}=\{c_i^1,\dots,c_i^{\Delta-2}\}$. The set $V_C\cup \bigcup_{i=2}^{x} V_C^i$ forms the stars-on-a-cycle graph. Note that $c_1$ is only connected to $c_2$ and $c_{x}$ in the stars-on-a-cycle graph. We have that the stars-on-a-cycle graph consists of $s=(x-1)\cdot (\Delta-1)+1\approx n/2$ nodes.
There are also nodes $V_L=\{\ell_1,\dots,\ell_{n-s}\}$, such that node $\ell_i$ is connected to $\ell_{i-1}$ and $\ell_{i+1}$ for $1<i<n/2$. The nodes $V_L$ forms the line and consists of $n-s\geq  n/2$ nodes. The node $c_1$ is connected to $\ell_1$.
There is an illustration of $G^{6,31}$ in Figure~\ref{fig:undir}.

We first argue that if at least one of $V_L'=\{\ell_{\lceil n/4\rceil},\dots,\ell_{n-s}\}$ is initially of type $t_1$, then with pr. lower bounded by a number depending only on $r$, type $t_1$ fixates (note that $|V_L'|\geq  n/4$ and thus, even if there is only a single node of type $t_1$ initially placed uniformly at random, it is in $V_L'$ with pr. $\geq 1/4$).

\begin{lemma}\label{lem:t_1fix}
With pr. above $\frac{1-1/r}{2}$ if at least one of $V_L'$ is initially of type $t_1$, then $t_1$ fixates.
\end{lemma}
The proof is based on applying the gambler's ruin twice. Once to find out that the pr. that $V_L$ eventually becomes all $t_1$ is above $\frac{1-1/r}{2}$ (it is nearly $1-1/r$ in fact) and once to find out that if $V_L$ is at some point all $t_1$, then the pr. that $t_2$ fixates is exponentially small with base $r$ and exponent $n-s$.
See the appendix for the proof. 

Whenever a node of $V_C^i$, for some $i$, changes type, we say that a {\em leaf-step} occurred. We will next consider the pr. that an effective step is a leaf-step.

\begin{lemma}\label{lem:leaf-step}
The pr. that an effective step is a leaf-step is at most $ \frac{r}{\Delta}$.
\end{lemma}
The proof is quite direct and considers that the probability that a leaf gets selected for reproduction over a center node in the stars-on-a-cycle graph.
See the appendix for the proof.

We are now ready for the theorem.

\begin{theorem}\label{thm:lbundirected}
Let $r>1$ be some fixed constant.
Consider $\Delta>2$ (the maximum degree of the graph), $n>4\Delta$ (sufficiently big), and some $k$ such that $0<k<n$. Then, if there are initially $k$ members of type $t_2$ placed uniformly at random, the expected fixation time of $G^{\Delta,n}$ is above $\frac{k\Delta(1-1/r)}{32r}$ effective steps.
\end{theorem}
\begin{proof}
Even if $k=n-1$, we have that with pr. at least $\frac{1}{4}$, the lone node of type $t_1$ is initially in $V_L'$. If so, by Lemma~\ref{lem:t_1fix}, type $t_1$ is going to fixate with pr. at least $\frac{1-1/r}{2}$. 
Note that even for $\Delta=3$, at least $\frac{n}{4}$ nodes of the graphs are in $V':=\bigcup_{i=2}^{x} V_C^i$ (i.e. the leaves of the stars-on-a-cycle graph). In expectation $\frac{k}{4}$ nodes of $V'$ are thus initially of type $t_2$. For fixation for $t_1$ to occur, we must thus make that many leaf-steps. Any effective step is a leaf-step with pr. at most $\frac{r}{\Delta}$ by Lemma~\ref{lem:leaf-step}. Hence, with pr. $\frac{1}{4}\cdot \frac{1-1/r}{2}$ ($\frac{1}{4}$ is the probability that at least one node of type $t_1$ is in $V_L'$ and $\frac{1-1/r}{2}$ is a lower bound on the fixation probability if a node of $V_L'$ is of type $t_1$) we must make $\frac{k\Delta}{4r}$ effective steps before fixation in expectation, implying that the expected fixation time is at least $\frac{k\Delta(1-1/r)}{32r}$ effective steps.
\end{proof}

\section{Sampling an effective step\label{sec:sample}}
In this section, we consider the problem of sampling an effective step. 
It is quite straightforward to do so in $O(m)$ time. 
We will present a data-structure that after $O(m)$ preprocessing can sample and 
update the distribution in $O(\Delta)$ time. 
For this result we assume that a uniformly random number can be selected between $0$ and $x$ 
for any number $x\leq n\cdot w(t)$ in constant time, a model that was also implicitly 
assumed in previous works~\cite{Diaz14}\footnote{The construction of~\cite{Diaz14} was to store a list for $t_1$ and a list for $t_2$ and then first decide if a $t_1$ or $t_2$ node would be selected in this step (based on $r$ and the number of nodes of the different types) and then pick a random such node. This works when all nodes of a type has the same weight but does not generalize to the case when each node can have a distinct weight based on the nodes successors like here}. 

\begin{remark}
If we consider a weaker model, that requires constant time for each random bit, then we need $O(\log n)$ random bits in expectation and additional $O(\Delta)$ 
amortized time, using a similar data-structure (i.e., a total of $O(\Delta+\log n)$ amortized 
time in expectation). The argument for the weaker model is presented in the Appendix. In this more restrictive model~\cite{Diaz14} would use $O(\log n)$ time per step for sampling.
\end{remark}

\subparagraph*{Sketch of data-structure}
We first sketch a list data-structure that supports 
(1)~inserting elements; (2)~removing elements; and (3)~finding a random element; 
such that each operation takes (amortized or expected) $O(1)$ time. 
The idea based on dynamic arrays is as follows:
\begin{enumerate}
\item  {\bf Insertion} Inserting elements takes $O(1)$ amortized time in a dynamic array, using the standard construction.  
\item  {\bf Deletion} Deleting elements is handled by changing the corresponding element to a null-value and then rebuilding the array, without the null-values, 
if more than half the elements have been deleted since the last rebuild. Again, this takes $O(1)$ amortized time. 
\item {\bf Find random element} Repeatedly pick a uniformly random entry. If it is not null, then  output it. Since the array is at least half full, 
this takes in expectation at most 2 attempts and thus expected $O(1)$ time.  
\end{enumerate}
At all times we keep a doubly linked list of empty slots, to find a slot for insertion in $O(1)$ time. 

\subparagraph*{Data-structure}
The idea is then as follows. We have $2\Delta$ such list data-structures, one for each pair of type and degree. 
We also have a weight associated to each list,  which is the sum of the weight of all nodes in the list, according to the modified dynamic evolution step. 
When the current type function is $\f$, we represent each node $v$ as follows: The corresponding list data-structure contains $|\Gamma_v(\f)|$ copies of $v$ (and $v$ keeps track of the locations in a doubly linked list).  Each node $v$ also keeps track of $\Gamma_v(\f)$, using another list data-structure. 
It is easy to construct the initial data-structure in $O(m)$ time (note: $\sum_v|\Gamma_v(\f)|\leq 2m$).

\subparagraph*{Updating the data-structure}
We can then update the data-structure when the current type function $\f$ changes to $\f[u\rightarrow t]$ 
(all updates have that form for some $t$ and $u$), by removing $u$ from the list data-structure $(\f(u),\deg u)$ 
containing it and adding it to $(t,\deg u)$. 
Note that if we removed $x'$ copies of $u$ from $(\f(u),\deg u)$ we add $\deg u-x'$ to $(t,\deg u)$. 
Also, we update each neighbor $v$ of $u$ (by deleting or adding a copy to $(\f(v),\deg v)$, 
depending on whether $\f(v)=t$). We also keep the weight corresponding to each list updated and $\Gamma_v(\f)$ 
for all nodes $v$. This takes at most $4\Delta$ data-structure insertions or deletions, and thus $O(\Delta)$ amortized time in total.

\subparagraph*{Sampling an effective step}
Let $\f$ be the current type function.
First, pick a random list $L$ among the $2\Delta$ lists, proportional to their weight.
 Then pick a random node $v$ from $L$. Then pick a node at random in $\Gamma_v(\f)$. This takes $O(\Delta)$ time in expectation.

\begin{remark}
Observe that picking a random list among the $2\Delta$ lists, proportional to their weight takes $O(\Delta)$ time to do naively: E.g. consider some ordering of the lists and let $w_i$ be the total weight of list $i$ (we keep this updated so it can be found in constant time). Pick a random number $x$ between 1 and the total weight of all the lists (assumed to be doable in constant time). Iterate over the lists in order and when looking at list $i$, check if $x< \sum_{j=1}^i w_j$. If so, pick list $i$, otherwise continue to list $i+1$. By making a binary, balanced tree over the lists (similar to what is used for the more restrictive model, see the Appendix), the time can be brought down to $O(\log \Delta)$ for this step - however the naive approach suffices for our application, because updates requires $O(\Delta)$ time.
\end{remark}

This leads to the following theorem.

\begin{theorem}\label{thm:sampling}
An effective step can be sampled in (amortized and expected) $O(\Delta)$ time after $O(m)$ preprocessing, if a uniformly random integer between $0$ and $x$, for any $0<x\leq n\cdot w(t)$, can be found in constant time. 
\end{theorem}

\newcommand{\distr}{\mathcal{D}}
\newcommand{\cale}{\mathcal{E}}
\newcommand{\calb}{\mathcal{B}}
\newcommand{\calf}{\mathcal{F}}
\newcommand{\calx}{\mathcal{X}}

\section{Algorithms for approximating fixation probability}\label{sec:algo}
We present the algorithms for solving 
the fixation, extinction, and generalized fixation problems.

\subparagraph*{The Meta-simulation algorithm}
Similar to~\cite{Diaz14}, the algorithms are instantiating the following meta-simulation 
algorithm, that takes a distribution over initial type functions $\distr$, type $t$ 
and natural numbers $u$ and $z$ as input:

\begin{function}[]
Let $y\leftarrow 0$\;
\For{$(i\in \{1,\dots,z\})$}{
Initialize a new simulation $I$ with initial type function $\f$ picked according to $\distr$\;
Let $j\leftarrow 0$\;
\While{$(I$ has not fixated$)$} {
\If {$(j\geq u)$} {
\Return Simulation took too long\;
}
Set $j\leftarrow j+1$\;
Simulate an effective step in $I$\;
}
\If {$(t$ fixated in $I)$}{
Set $y\leftarrow y+1$\; 
}}
\Return $y/z$\;
\caption{MetaSimulation($t$,$z$,$u$,$\distr$)}
\end{function}

\subparagraph*{Basic principle of simulation}
Note that the meta-simulation algorithm uses $O(uz\Delta)$ time (by Theorem~\ref{thm:sampling}). In essence, the algorithm runs $z$ simulations of the process and terminates with ``Simulation took too long'' 
iff some simulation took over $u$ steps. Hence, whenever the algorithm returns a number it is the mean of $z$ binary random variables, 
each equal to~1 with probability $\Pr[\calf_t\mid \cale_u]$, where $\calf_t$ is the event that $t$ fixates and $\cale_u$ is the event 
that fixation happens before $u$ effective steps, when the initial type function is picked according to $\distr$ (we note that the 
conditional part was overlooked in~\cite{Diaz14}, moreover, instead of steps we consider only effective steps).
By ensuring that $u$ is high enough and that the approximation is tight enough (basically, that $z$ is high enough), 
we can use $\Pr[\calf_t\mid \cale_u]$ as an approximation of $\Pr[\calf_t]$, as shown in the following lemma.

\begin{lemma}\label{lem:meta_condition}
Let $0<\epsilon<1$ be given. Let $\calx,\cale$ be a pair of events and $x$ a number, such that  
$\Pr[\cale]\geq 1-\frac{\epsilon\cdot \Pr[\calx]}{4}$ and that $x\in [(1-\epsilon/2)\Pr[\calx\mid \cale],(1+\epsilon/2)\Pr[\calx\mid \cale]]$. 
Then
\[
x\in [(1-\epsilon)\cdot \Pr[\calx],(1+\epsilon)\cdot \Pr[\calx]]\enspace .
\]
\end{lemma}

\subparagraph*{The value of $\bm{u}$: $\bm{u_{z,r}}$}
Consider some fixed value of $z$.
The value of $u$ is basically just picked so high that $\Pr[\cale_u]\geq 1-\frac{\epsilon\cdot \Pr[\calf_t]}{4}$ 
(so that we can apply Lemma~\ref{lem:meta_condition}) and such that after taking union bound over the $z$ trials, 
we have less than some constant probability of stopping. 
The right value of $u$ is thus sensitive to $r$, but in all cases at most 
$O(n^2 \Delta^2\max(\log z,\log \epsilon^{-1}))$, because of Theorem~\ref{thm:fixationtime}.
More precisely, we let 
\[
u_{z,r}=\begin{cases}
30n\cdot \max(\log z,\log \epsilon^{-1}) &\ift r\geq 2\Delta\\
\frac{30n\Delta}{\min(r-1,1)}\cdot \max(\log z,\log \epsilon^{-1}) &\ift 1+\frac{1}{n\cdot \Delta}\leq r< 2\Delta\\
20n^2\Delta^2\cdot \max(\log z,\log \epsilon^{-1}) &\ift r<1+\frac{1}{n\cdot \Delta}\enspace .
\end{cases}
\]

\subparagraph*{Algorithm $\bm{\algo1}$}
We consider the fixation problem for $t_1$.
Algorithm $\algo1$ is as follows:
\begin{enumerate}
\item Let $\distr$ be the uniform distribution over the $n$ type functions where exactly one node is $t_1$.
\item Return MetaSimulation($t_1$,$z$,$u_{z,r}$,$\distr$), for $z=48\cdot \frac{n}{\epsilon^2}$.
\end{enumerate}

\subparagraph*{Algorithm $\bm{\algo2}$}
We consider the extinction problem for  $t_1$. Algorithm $\algo2$ is as follows:
\begin{enumerate}
\item Let $\distr$ be the uniform distribution over the $n$ type functions where exactly one node is $t_2$.
\item 
Return MetaSimulation($t_1$,$z$,$u_{z,r}$,$\distr$), for $z=24/\epsilon^{2}$.
\end{enumerate}

\subparagraph*{Algorithm $\bm{\algo3}$}
We consider the problem of (additively) approximating the fixation probability given some 
type function $\f$ and type $t$.
Algorithm $\algo3$ is as follows:
\begin{enumerate}
\item Let $\distr$ be the distribution that assigns $1$ to $\f$.
\item Return MetaSimulation($t$,$z$,$u_{z,r}$,$\distr$), for $z=6/\epsilon^{2}$.
\end{enumerate}

\begin{theorem}\label{thm:FPRAS}
Let $G$ be a connected undirected graph of $n$ nodes with the highest degree $\Delta$, 
divided into two types of nodes $t_1,t_2$, such that $r=w(t_1) > w(t_2)=1$. 
Given $\frac{1}{2}>\epsilon>0$, let $\alpha=n^2\cdot \Delta\cdot  \epsilon^{-2} \cdot \max(\log n,\log \epsilon^{-1})$
and $\beta= n\cdot \Delta\cdot \epsilon^{-2}\cdot\log \epsilon^{-1}$. Consider the running times:
\[
T(x)=\begin{cases}
O(x) &\ift r\geq 2\Delta\\
O(\frac{x \cdot \Delta}{\min(r-1,1)}) &\ift 1+\frac{1}{n\cdot \Delta}\leq r< 2\Delta\\
O(n\cdot \Delta^2 \cdot x) & \ift 1<r<1+\frac{1}{n\cdot \Delta} \enspace .
\end{cases}
\]
\begin{itemize}
\item {\bf Fixation (resp. Extinction) problem for $\bm{t_1}$} 
Algorithm $\algo1$ (resp. $\algo2$) is an FPRAS algorithm, with running time $T(\alpha)$ (resp. $T(\beta)$), 
that with probability at least $\frac{3}{4}$ outputs a number in $[(1-\epsilon)\cdot \rho,(1+\epsilon)\cdot \rho]$, 
where $\rho$ is the solution of the fixation (resp. extinction) problem for $t_1$.

\item {\bf Generalized fixation problem} 
Given an initial type function $\f$ and a type $t$, there is an (additive approximation) algorithm, $\algo3$, with running time $T(\beta)$, 
that with probability at least $\frac{3}{4}$ outputs a number in $[\rho-\epsilon,\rho+\epsilon]$, 
where $\rho$ is the solution of the generalized fixation problem given $\f$ and $t$.
\end{itemize}
\end{theorem}

\begin{remark}
There exists no known FPRAS for the generalized fixation problem and since the fixation probability might be exponentially small such an algorithm might not exist. (It is exponentially small for fixation of $t_2$, even in the Moran process (that is, when the graph is complete) when there initially is 1 node of type $t_2$)
\end{remark}
 
\subparagraph*{Alternative algorithm for extinction for $t_2$}
We also present an alternative algorithm for extinction for $t_2$ when $r$ is big.
This is completely different from the techniques of~\cite{Diaz14}. 
The alternative algorithm is based on the following result where we show for
big $r$ that $1/r$ is a good approximation of the extinction probability for $t_2$, 
and thus the algorithm is polynomial even for big $r$ in binary.

\begin{theorem}\label{thm:too_big_r}
Consider an undirected graph $G$ and consider the extinction problem for $t_2$ on $G$.
If $r\geq \max(\Delta^2,n)/\epsilon$,
then $\frac{1}{r}\in [(1-\epsilon)\cdot \rho,(1+\epsilon)\cdot \rho]$, where $\rho$ is the solution of the extinction problem for $t_2$.
\end{theorem}

\subparagraph*{Proof sketch}
We present a proof sketch, and details are in the Appendix. We have two cases:
\begin{itemize}
\item By \cite[Lemma 4]{Diaz14}, we have $\rho\geq \frac{1}{n+r}$.
Thus, $(1+\epsilon)\cdot \rho\geq \frac{1}{r}$, as desired, since $n/\epsilon\leq r$.

\item On the other hand, the probability of fixation for $t_2$ in the first effective step is at most $\frac{1}{r+1}<\frac{1}{r}$ 
(we show this in Lemma~\ref{lem:extinct_in_1} in the Appendix). The probability that fixation happens for $t_2$ after the first effective 
step is at most $\epsilon/r$ 
because of the following reason:
By Lemma~\ref{lem:large_r}, the probability of increasing the number of members of $t_2$ is at most 
$p:=\frac{1}{r/\Delta+1}$ and otherwise it decrements. We then model the problem as a Markov chain $M$ with state space corresponding to the number of members of $t_1$, 
using $p$ as the probability to decrease the current state. 
In $M$ the starting state is state~2 (after the first effective step, if fixation did not happen, then the number of members of $t_1$ is 2). 
Using that $\Delta^2/\epsilon\leq r$, we see that the probability of absorption in state~$0$ of $M$ from state~2 is less than $\epsilon/r$. 
Hence, $\rho$ is at most $(1+\epsilon)/r$ and $(1-\epsilon)\rho$ is thus less than $1/r$.
\end{itemize}\qed

\begin{remark}
While Theorem~\ref{thm:too_big_r} is for undirected graphs, a variant (with larger $r$ and which requires the computation of the pr. that $t_1$ goes extinct in the first step) can be established even for directed graphs, see the Appendix.
\end{remark}

\subparagraph*{Concluding remarks}
In this work we present faster Monte-Carlo algorithms for approximating fixation 
probability for undirected graphs (see Remark~\ref{rem:appendix} in the Appendix for detailed comparison). 
An interesting open question is whether the fixation probability can be approximated
in polynomial time for directed graphs.

\clearpage
\bibliographystyle{abbrv}
\bibliography{bibliography}

\begin{thebibliography}{10}

\bibitem{ACN15}
B.~Adlam, K.~Chatterjee, and M.~Nowak.
\newblock Amplifiers of selection.
\newblock In {\em Proc. R. Soc. A}, volume 471, page 20150114. The Royal
  Society, 2015.

\bibitem{Hauert14}
F.~D\'ebarre, C.~Hauert, and M.~Doebeli.
\newblock Social evolution in structured populations.
\newblock {\em Nature Communications}, 2014.

\bibitem{Daz13}
J.~D{\'i}az, L.~A. Goldberg, G.~B. Mertzios, D.~Richerby, M.~Serna, and P.~G.
  Spirakis.
\newblock On the fixation probability of superstars.
\newblock {\em Proceedings of the Royal Society A: Mathematical, Physical and
  Engineering Science}, 469(2156), 2013.

\bibitem{Diaz14}
J.~D{\'i}az, L.~A. Goldberg, G.~B. Mertzios, D.~Richerby, M.~Serna, and P.~G.
  Spirakis.
\newblock {Approximating Fixation Probabilities in the Generalized Moran
  Process}.
\newblock {\em Algorithmica}, 69(1):78--91, 2014 (Conference version SODA
  2012).

\bibitem{Ewens04}
W.~Ewens.
\newblock {\em {Mathematical Population Genetics 1: I. Theoretical
  Introduction}}.
\newblock Interdisciplinary Applied Mathematics. Springer, 2004.

\bibitem{Frean07}
M.~Frean, P.~B. Rainey, and A.~Traulsen.
\newblock The effect of population structure on the rate of evolution.
\newblock {\em Proceedings of the Royal Society B: Biological Sciences},
  280(1762), 2013.

\bibitem{ICALP16}
A.~Galanis, A.~G{\"o}bel, L.~A. Goldberg, J.~Lapinskas, and D.~Richerby.
\newblock {Amplifiers for the Moran Process}.
\newblock In {\em 43rd International Colloquium on Automata, Languages, and
  Programming (ICALP 2016)}, volume~55, pages 62:1--62:13, 2016.

\bibitem{Nowak10}
A.~L. Hill, D.~G. Rand, M.~A. Nowak, and N.~A. Christakis.
\newblock {Emotions as infectious diseases in a large social network: the SISa
  model}.
\newblock {\em Proceedings of the Royal Society B: Biological Sciences},
  277:3827--3835, 2010.

\bibitem{ICN15}
R.~Ibsen-Jensen, K.~Chatterjee, and M.~A. Nowak.
\newblock Computational complexity of ecological and evolutionary spatial
  dynamics.
\newblock {\em Proceedings of the National Academy of Sciences},
  112(51):15636--15641, 2015.

\bibitem{Karlin75}
S.~Karlin and H.~M. Taylor.
\newblock {\em {A First Course in Stochastic Processes, Second Edition}}.
\newblock Academic Press, 2 edition, Apr. 1975.

\bibitem{Nowak05}
E.~Lieberman, C.~Hauert, and M.~A. Nowak.
\newblock {Evolutionary dynamics on graphs}.
\newblock {\em Nature}, 433(7023):312--316, Jan. 2005.

\bibitem{Moran62}
P.~A.~P. Moran.
\newblock {\em {The Statistical Processes of Evolutionary Theory}}.
\newblock Oxford University Press, Oxford, 1962.

\bibitem{NowakBook}
M.~A. Nowak.
\newblock {\em {Evolutionary Dynamics: Exploring the Equations of Life}}.
\newblock Harvard University Press, 2006.

\bibitem{Nowak92}
M.~A. Nowak and R.~M. May.
\newblock Evolutionary games and spatial chaos.
\newblock {\em Nature}, 359:826, 1992.

\bibitem{Ohtsuki06}
H.~Ohtsuki, C.~Hauert, E.~Lieberman, and M.~A. Nowak.
\newblock {A simple rule for the evolution of cooperation on graphs and social
  networks}.
\newblock {\em Nature}, 441:502--505, 2006.

\bibitem{Diaz16}
M.~Serna, D.~Richerby, L.~A. Goldberg, and J.~D{\'i}az.
\newblock Absorption time of the moran process.
\newblock {\em Random Structures \& Algorithms}, 48(1):137--159, 2016.

\bibitem{Shakarian12}
P.~Shakarian, P.~Roos, and A.~Johnson.
\newblock A review of evolutionary graph theory with applications to game
  theory.
\newblock {\em Biosystems}, 107(2):66 -- 80, 2012.

\end{thebibliography}
\clearpage

\section*{Appendix}

\section{Details of Section~\ref{sec:ineffective}}
In this section, we prove Lemma~\ref{lem:preserve} and Theorem~\ref{thm:fixationtime}. 
\begin{replemma}{lem:preserve}
Fix any type function $\f$ such that neither type has fixated. 
Let $\f_d$ (resp., $\f_m$) be the next type function under dynamic evolution step (resp., modified dynamic evolution step). 
Then, $\Pr[\f\neq \f_d]> 0$ and for all type functions $\f'$ we have: 
$\Pr[\f'=\f_d\mid \f\neq \f_d]=\Pr[\f'=\f_m]
$.
\end{replemma}
\begin{proof}
Let $\f,\f_d,\f_m$ be as in the lemma statement. 
First note that $\Pr[\f_d\neq \f]> 0$ since $G$ is connected and no type has fixated in $\f$. Therefore, there must be some edge $(u,v)\in E$, such that $\f(u)\neq \f(v)$. There is a positive probability that $u$ and $v$ is selected by the unmodified evolution step, in which case $\f_d\neq \f$. Thus $\Pr[\f_d\neq \f]> 0$. 

We consider the node selection in the two cases:
\begin{enumerate}
\item Let $v$ be the node picked in the first part of the unmodified dynamic evolution step and $u$ be the node picked in the second part. 
\item Similarly, let $v'$ be the node picked in the first part of the modified dynamic evolution step and $u'$ be the node picked in the second part. 
\end{enumerate}
Observe that $(u,v),(u',v')\in E$.
We will argue that for all $(u'',v'')\in E$, we have that 
\[
\Pr[u=u''\cup v=v''\mid \f(u)\neq \f(v)]=\Pr[u'=u''\cup v'=v''].
\] 
This clearly implies the lemma statement, since $\f_d=\f[u\rightarrow \f(v)]$ and $\f_m=\f[u'\rightarrow \f(v')]$, 
by the last part of the unmodified and modified dynamic evolution step. 
If $\f(u'')=\f(v'')$, then $\Pr[u=u''\cup v=v''\mid \f(u)\neq \f(v)]=0$, because $\f(u'')=\f(u)\neq \f(v)=\f(v'')$ which contradicts that $\f(u'')=\f(v'')$. Also, if $\f(u'')=\f(v'')$, then $\Pr[u'=u''\cup v'=v'']=0$, because $\f(u'')=\f(u')\neq \f(v')=\f(v'')$ (note $\f(u')\neq \f(v')$ because $u'$ was picked from $\Gamma_{v'}$), again contradicting that $\f(u'')=\f(v'')$.

We therefore only need to consider that $(u'',v'')\in E$ and $\f(u'')\neq \f(v'')$.
The probability to pick $v$ and then pick $u$ in an unmodified dynamic evolution step is $\frac{w(\f(v))}{W(\f)}\cdot \frac{1}{\deg v}$, for any $(u,v)\in E$ and especially the ones for which $\f(u)\neq \f(v)$. Hence, $\Pr[u=u''\cup v=v'']=\frac{\f(v'')}{W(\f)}\cdot \frac{1}{\deg v''}$. Thus, also, $\Pr[(u=u''\cup v=v'')\bigcap (\f(u)\neq \f(v))]=\frac{w(\f(v''))}{W(\f)}\cdot \frac{1}{\deg v''}$.
We also have that \[
\Pr[\f(u)\neq \f(v)]=\sum_{\substack{(u,v)\in E\\\f(u)\neq \f(v)}}\frac{w(\f(u))}{W(\f)}\cdot \frac{1}{\deg u} =\sum_{v\in V}\frac{w(\f(v))}{W(\f)}\cdot \frac{1}{\deg v}\cdot |\Gamma_v(\f)|=\frac{W'(\f)}{W(\f)} \enspace , 
\]
where the second equality comes from that $\Gamma_v(\f)$ is the set of nodes $u$ such that $(u,v)\in E$ and $\f(u)\neq \f(v)$.
Thus, \[\Pr[u=u''\cup v=v''\mid \f(u)\neq \f(v)]=\frac{\Pr[(u=u''\cup v=v'')\bigcap (\f(u)\neq \f(v))]}{\Pr[\f(u)\neq \f(v)]}=\frac{w(\f(v''))}{W'(\f)}\cdot \frac{1}{\deg v''}\]
The probability to pick $v'$ and then pick $u'$ in a modified dynamic evolution step, for some $(u',v')\in E$ is \[\frac{w(\f(v'))\cdot \frac{|\Gamma_{v'}(\f)|}{\deg v'}}{W'(\f)}\cdot \frac{1}{|\Gamma_{v'}(\f)|}=\frac{w(\f(v'))}{W'(\f)}\cdot \frac{1}{\deg v'} \enspace .\]
Hence, \[
\Pr[u'=u''\cup v'=v'']=\frac{w(\f(v''))}{W'(\f)}\cdot \frac{1}{\deg v''}=\Pr[u=u''\cup v=v''\mid \f(u)\neq \f(v)] \enspace .
\]
This completes the proof of the lemma.
\end{proof}
Next, the proof of Theorem~\ref{thm:fixationtime}.
\begin{reptheorem}{thm:fixationtime}
Let $t_1$ and $t_2$ be the two types, such that $r=w(t_1) >  w(t_2)=1$. 
Let $\Delta$ be the maximum degree.  Let $k$ be the number of nodes of type $t_2$ in the initial type function. 
The following assertions hold:
\begin{itemize}
\item {\em Bounds dependent on $r$}
\begin{enumerate}
\item {\em Expected steps} The process requires at most $3k\Delta/\min(r-1,1)$ effective steps in expectation, before fixation is reached.
\item {\em Probability} For any integer $x\geq 1$, after $6xn\Delta/\min(r-1,1)$ effective steps, the probability that the process has not 
fixated is at most $2^{-x}$, irrespective of the initial type function.
\end{enumerate}

\item {\em Bounds independent on $r$}
\begin{enumerate}
\item  {\em Expected steps}
The process requires at most $2nk\Delta^2$ effective steps in expectation, before fixation is reached.
\item  {\em Probability}
For any integer $x\geq 1$, after $4 x n^2 \Delta^2$ effective steps, the probability that the process has not fixated is at most $2^{-x}$,  irrespective of the initial type function.
\end{enumerate}

\item {\em Bounds for $r\geq 2\Delta$} 
\begin{enumerate}
\item {\em Expected steps} The process requires at most $3k$ effective steps in expectation, before fixation is reached.
\item  {\em Probability} For any integer $x\geq 1$, after $6 x n$ effective steps, the probability that the process has not fixated is at most $2^{-x}$,  irrespective of the initial type function.
\end{enumerate}
\end{itemize}
\end{reptheorem}
\begin{proof}
Observe that if $k=0$, then fixation has been reached in 0 (effective) steps. Thus assume $k\geq 1$.
We first argue about the first item of each case (i.e., about expected steps) and then about
the second item (i.e., probability) of each case.

\subparagraph*{Expected steps of first item}
In every step, for $r>1$, the potential increases by $-k_2=\frac{r-1}{\Delta\cdot (r+1)}$, unless fixation has been achieved, in expectation by~Lemma~\ref{lem:fast1step}. 
Let $n'=\sum_{v\in V} \frac{1}{\deg v}$ be the maximum potential. Let $\f$ be the initial type function. Observe that the potential initially is $\psi(\f)=\sum_{v\in V_{t_1,\f}}\frac{1}{\deg v}=\sum_{v\in V}\frac{1}{\deg v}-\sum_{v\in V_{t_2,\f}}\frac{1}{\deg v}\geq n'-k$. Hence, the potential just needs to increase by at most $k$.  
Similar to \cite{Diaz14}, we consider a modified process where if $t_2$ fixates, then the next type function $\f'$ has a uniformly random node mapped to $t_1$ (note that this only increases the time steps till $t_1$ fixates). We see that in every step, unless fixation happens for $t_1$, the potential increases by $-k_2$ (in case fixation for $t_2$ has happened in the previous step, the potential increases by at least $\frac{1}{\Delta}> -k_2$, since $\frac{r-1}{r+1}<1$).
Applying \cite[Theorem~6]{Diaz14}, with $k_1=k$ and $-k_2$, we have that the number of effective steps in expectation, when starting with the type function $\f$, is at most $\frac{k\Delta(r+1)}{(r-1)}$, for the modified process. Thus the upper bound also follows for the original process, which fixates faster. Note that for $r\geq 2$ we have that $\frac{r+1}{r-1}\leq 3$ and that for $1<r<2$ we have that $\frac{r+1}{r-1}<\frac{3}{r-1}$. Thus, $\frac{r+1}{r-1}\leq \frac{3}{\min(r-1,1)}$. This establishes the expected number of effective steps for the first item.

\subparagraph*{Expected steps of second item}
The fact that $2nk \Delta^2$ effective steps is sufficient in expectation for $r$ close to 1 can be seen as follows.
Let $\f^*$ denote the type function that assigns all nodes to $t_1$.
Let $\f$ be the initial type function, and we denote by $\overline{\f}$ its complement, 
i.e., $\overline{\f}$ maps every node to the opposite type as compared to $\f$.
Similar to the proof\footnote{Note that the proof of \cite[Theorem~11]{Diaz14} does not directly use that $r=1$, even though the statement they end up with does} of \cite[Theorem~11]{Diaz14}
we get that $\Delta^2((\psi(\f^*))^2-(\psi(\f))^2)$ is sufficient in expectation, where $\f$ is the initial type function.
The change as compared to the proof of \cite[Theorem~11]{Diaz14} consists of using that $\E[(\psi(\f')-\psi(\f''))^2]\geq \Delta^{-2}$  instead of \cite[Equation~(3)]{Diaz14}, 
where $\f'$ is a fixed type function and $\f''$ the following type function according to a modified dynamic evolution step. 
Let $a$ be $\psi(\f)$ and $b$ be $\psi(\f^*)-\psi(\f)=\psi(\overline{\f})$. 
Hence, $\Delta^2((\psi(\f^*))^2-(\psi(\f))^2)=\Delta^2((a+b)^2-a^2)=\Delta^2(2ab+b^2)$. This is monotone increasing in $a$ and $b$ (since, $a,b,\Delta$ are positive). 
For all type functions $\widehat{\f}$, we have that $\psi(\widehat{\f})$ is at most the number of nodes assigned to $t_1$ by $\widehat{\f}$.
Hence, we have that $a=\psi(\f)\leq n-k$ and $b=\psi(\overline{\f})\leq k$. 
Thus \[
\Delta^2((\psi(\f^*))^2-(\psi(\f))^2)=\Delta^2(2ab+b^2)\leq \Delta^2(2(n-k)k+k^2)=\Delta^2(2nk-k^2)\leq 2nk\Delta^2 \enspace ,
\]
as desired.

\subparagraph*{Expected steps of third item}
We can consider the potential function $\psi'$, which is $\psi'(\f')=|V_{t_1,\f'}|$. We see that in every effective step the potential increases by at least $\frac{1}{3}$ in expectation, by Lemma~\ref{lem:large_r}. Similar to the first item, we can then define a modified process and apply \cite[Theorem~6]{Diaz14}, and see that $3k$ effective steps 
suffices in expectation.

\subparagraph*{Probability bounds of the three cases}
Follows from the above three items and Lemma~\ref{lemm:conc}.
\end{proof}

\section{Details of Section~\ref{sec:lbundirected}}
In this section we prove Lemma~\ref{lem:t_1fix} and Lemma~\ref{lem:leaf-step}.

\begin{replemma}{lem:t_1fix}
With pr. above $\approx \frac{1-1/r}{2}$ if at least one of $V_L'$ is initially of type $t_1$, then $t_1$ fixates.
\end{replemma}
\begin{proof}
In this proof we consider that there is initially 1 member of $t_1$. As shown by \cite{Diaz16}, the fixation probability is increasing in the number of members.

If the graph just consisted of the nodes in $V_L$ (and $v_1$ was only connected to $v_2$), then if the initial state had a single member of type $t_1$, the pr. that $t_1$ fixated would be $\frac{1-1/r}{1-r^{x-n}}\approx 1-1/r$ (using that it is in essence the gambler's ruin problem when the probability for winning each round for $t_2$ is $\frac{1}{r+1}$, but $t_2$ starts with $n-x-1$ pennies compared to 1 for $t_1$). 
We see that in the original graph, the only difference is that $v_1$ is connected to the stars-on-a-cycle graph. Thus, we see that as long as $v_1$ is always a member of $t_2$, there are no differences. 
We can use that as follows: If in the initial type function the lone member of $t_1$ is in $V_L'$, then with pr. above $1/2$, the node $v_{n-x}$ becomes of type $t_1$ before $v_1$, because at all times (as long as neither $v_1$ or $v_{n-x}$ has ever been of type $t_1$), the nodes of type $t_1$ forms an interval $\{i,\dots,j\}$ and the pr. that in the next iteration $v_{i-1}$ becomes of type $t_1$ is equal to the pr. that in the next iteration $v_{j+1}$ becomes of type $t_1$. Hence, we get that with pr. above $\frac{1-1/r}{2(1-r^{x-n})}\approx \frac{1-1/r}{2}$ at some point the members of $t_1$ is exactly $V_L$. 

But if $t_2$ fixates after all members of $V_L$ has become of type $t_1$, $V_L$ must thus eventually go from being all $t_1$ to being all $t_2$. We will now argue that it is exponentially unlikely to happen. For $V_L$ to change to being all $t_2$, we must have that in some step, $v_1$ has become a member of type $t_2$. 
We can now consider a stronger version of $t_2$ (that thus have more pr. to fixate), where $v_1$ can only become of type $t_1$ if it was reproduced to from $v_2$. We will say that $t_2$ starts an attempt whenever $v_1$ becomes of type $t_2$. We say that $t_2$ wins the attempt if $V_L$ eventually becomes only of type $t_2$ and we say that $t_2$ loses the attempt if $V_L$ becomes of type $t_1$. 
Note that $t_2$ needs to win an attempt to fixate.
Again, using gambler's ruin, we see that $t_2$ wins any attempt with pr. $1-\frac{1-r^{x-n-1}}{1-r^{x-n}}=1-\frac{r^{n-x}-r^{-1}}{r^{n-x}-1}\approx \frac{1}{r^{n-x-1}}$. But the process fixates, as we showed earlier, in $3k\Delta/\min(r-1,1)$ effective steps, which is then an upper bound on the number of attempts. But then the pr. that $t_2$ ever wins an attempt is at most \[
3k\Delta/\min(r-1,1)\cdot (1-\frac{r^{n-x}-r^{-1}}{r^{n-x}-1})\approx \frac{3k\Delta}{r^{n-x}}\enspace ,\]
 which is exponentially small for fixed $r$ and we thus see that the pr. that $t_1$ eventually has all of $V_L$ (if one of $\{\ell_{\lceil n/4\rceil},\dots,\ell_{n-x}\}$ is initially of type $t_1$) is approximately the pr. that $t_1$ fixates. 
\end{proof}

Recall that whenever a node of $V_C^i$, for some $i$, changes type, we say that a leaf-step occurred. 

\begin{replemma}{lem:leaf-step}
The pr. that an effective step is a leaf-step is at most $ \frac{r}{\Delta}$.
\end{replemma}
\begin{proof}
Consider some initial type function $\f$. We will consider the conditional probability of a leaf-step conditioned on a node (but not $c_1$) in the stars-on-a-cycle graph is selected for reproduction (clearly, if no such node is selected for reproduction, the pr. of a leaf-step is 0). For simplicity, we will furthermore condition on the node that gets selected for reproduction is in $\{c_i\}\cup V_C^i$ for some $i$ (this is without loss of generality because all nodes in the stars-on-a-cycle graph is in one of those sets). Let $t:=\f(c_i)$ and let $t'$ be the other type. We will consider that $x>0$ nodes of $V_C^i$ is of type $t'$ (if no nodes of $V_C^i$ is of type $t'$ then the pr. of a leaf-step, conditioned on a node of $\{c_i\}\cup V_C^i$ being selected for reproduction is 0). For a leaf step to occur, we need to do as follows: we must select $c_i$ for reproduction and we must select one of the nodes of $V_C^i$ for death. Let $x'\in \{x,x+1,x+2\}$ (the +2 is because of the other center nodes) be the number of neighbors of $c_i$ which are of type $t'$.
The fitness of $c_i$ is then $w(t)\cdot \frac{x'}{\Delta}$. The fitness of each node of type $t'$ in $V_C^i$ is $w(t')$. If $c_i$ is picked, the pr. of a leaf-step is $\frac{x}{x'}$ (because if a center neighbor is selected, no leaf-step occurred). The pr. of a leaf step is then 
\[\frac{w(t)\cdot \frac{x'}{\Delta}\cdot \frac{x}{x'}}{xw(t')+w(t)}=\frac{w(t)}{\Delta(w(t')+w(t)/x)}\leq \frac{w(t)}{\Delta w(t')}
\]
Since in the worst case $w(t)=r$ and $w(t')=1$, we get our result.
\end{proof}

\section{Details of Section~\ref{sec:sample}}
In this section we consider the weaker model, where getting even a single random bit costs constant time.

We can use the same data-structure as in Section~\ref{sec:sample}, since the data-structure is deterministic.
Thus, we just need to argue how to sample an effective step given the data-structure. (i.e. like in the paragraph {\bf Sampling an effective step} of Section~\ref{sec:sample}).

There are thus 3 things we need to consider, namely (1)~picking random list $L$ among the $2\Delta$ lists, proportional to their weight; (2) pick a random node $v$ from $L$ and (3) pick a node at random in $\Gamma_v(\f)$.

Picking a random list among the $2\Delta$ lists, proportional to their weight takes $O(\Delta)$ time and uses $O(\log \Delta)$ random bits, e.g. as follows:
We can build a binary, balanced tree where each leaf corresponds to a list. We can then annotate each node of the tree with 
the sum of the weights of all lists below it. This can be done in $O(\Delta)$ time. 
Afterwards, we can then find a random list proportional to its weight using $O(\log \Delta)$ bits in expectation as follows: Start at the root of the tree.
\begin{enumerate}
\item When looking at an internal node $v$ do as follows:
Let $y$ be the number annotated on $v$ and let $z$ be the number annotated on the left child.
 Construct a random number $x$ bit for bit, starting with the most significant, between 0 and $w=2^{\lceil\log y\rceil}$. Let $x_{i}$ be the random variable denoting $x$ after $i$ bits have been added.
 Let $x_{i,1}$ be $x_{i}+w-2^{i}-1$, i.e. the greatest number $x$ can become if the remaining bits are all~1.
 If, after having added the $i$'th bit, $x_{i,1}\leq z$, continue to the left child.
 If, after having added the $i$'th bit, $x_i>z$ and $x_{i,1}\leq y$ continue to the right child.
 If, after having added the $i$'th bit, $x_i>y$, start over at node $v$.
 Otherwise, add bit $i+1$ to $x$. 
 
 When adding a bit to $x$ in step $i$, at least one of 
  $x_{i,1}\leq z$ or $x_i>z$ becomes satisfied with pr. 1/2. Similar, at least one of $x_{i,1}\leq y$ or $x_i>y$ becomes satisfied with pr. 1/2. Thus, in expectation, we need four steps before we continue to either the left, the right or back to $v$. Since $2^{\lceil\log y\rceil}<2y$, we have that we start over at $v$ with pr. at most 1/2. Thus, in expectation, after less than 8 steps, we go to either the left or the right child. Hence, after $O(\log \Delta)$ steps we reach a list.
 
 \item When looking at a list, output it.
\end{enumerate}

We can find a random element in the atleast half-full list $L$ by repeatedly finding a random number $x$ between $0$ and $2^{\lceil \log s\rceil}$, where $s\leq n$ is the size of $L$. If $x$ is below $s$ and is in $L$, output it, otherwise, find the next random number. 

Note that $2^{\lceil \log s\rceil}<2s$. Thus, $x\leq s$ with pr. 1/2. If so, we have another $1/2$ that  entry $x$ will be filled. Hence, we need to pick 4 numbers in expectation. Note that this thus takes $O(\log n)$ time/random bits.

We can do similar to find an element in $\Gamma_v(\f)$ in $O(\log (\Delta))$ time/random bits.

This leads to the following theorem.

\begin{theorem}
If no more than one random bit can be accessed in constant time, then an effective step can be sampled in (amortized and expected) $O(\Delta+\log n)$ time, using $O(\log n)$ expected random bits, after deterministic $O(m)$ preprocessing.
\end{theorem}

\section{Details of Section~\ref{sec:algo}}
In this section, we prove Lemma~\ref{lem:meta_condition} and give the correctness and time-bound arguments for our algorithms, which implies Theorem~\ref{thm:FPRAS}.
\begin{replemma}{lem:meta_condition}
Let $0<\epsilon<1$ be given. Let $\calx,\cale$ be a pair of events and $x$ a number, such that  
$\Pr[\cale]\geq 1-\frac{\epsilon\cdot \Pr[\calx]}{4}$ and that $x\in [(1-\epsilon/2)\Pr[\calx\mid \cale],(1+\epsilon/2)\Pr[\calx\mid \cale]]$. 
Then
\[
x\in [(1-\epsilon)\cdot \Pr[\calx],(1+\epsilon)\cdot \Pr[\calx]]\enspace .
\]
\end{replemma}
\begin{proof}
Fix $0<\epsilon<1$. Let $\calx,\cale$ be a pair of events, such that 
$\Pr[\cale]\geq 1-\frac{\epsilon\cdot \Pr[\calx]}{4}$ and that $x\in [(1-\epsilon/2)\Pr[\calx\mid \cale],(1+\epsilon/2)\Pr[\calx\mid \cale]]$.

We have that \begin{align*}
\Pr[\calx]&=\Pr[\cale]\cdot\Pr[\calx\mid \cale]+(1-\Pr[\cale])\cdot\Pr[\calx\mid \neg \cale]\\ 
&\geq (1-\frac{\epsilon\cdot \Pr[\calx]}{4})\cdot \Pr[\calx\mid \cale]\\
&\geq \Pr[\calx\mid \cale]-\frac{\epsilon\cdot \Pr[\calx]}{4}\enspace .
\end{align*}
Also,
\begin{align*}
\Pr[\calx]&=\Pr[\cale]\cdot\Pr[\calx\mid \cale]+(1-\Pr[\cale])\cdot\Pr[\calx\mid \neg \cale] \\ 
& \leq \Pr[\calx\mid \cale]+\frac{\epsilon\cdot \Pr[\calx]}{4}\cdot 
\Pr[\calx\mid \neg \cale]\\
&\leq \Pr[\calx\mid \cale]+\frac{\epsilon\cdot \Pr[\calx]}{4} \enspace .
\end{align*}
We now see that \begin{align*}
&x\in [(1-\epsilon/2)\Pr[\calx\mid \cale],(1+\epsilon/2)\Pr[\calx\mid \cale]]\\
\Rightarrow &x\in [(1-\epsilon/2)(1-\epsilon/4)\Pr[\calx],(1+\epsilon/2)(1+\epsilon/4)\Pr[\calx]]
\end{align*}
Note that
\[
(1-\epsilon/2)\cdot (1-\epsilon/4)=1+\epsilon^2/8-3\epsilon/4>1-\epsilon
\] and \[
(1+\epsilon/2)\cdot (1+\epsilon/4)=1+\epsilon^2/8+3\epsilon/4<1+7\epsilon/8<1+\epsilon\enspace ,
\] using that $0<\epsilon^2<\epsilon<1$. Hence, \[
 x\in [(1-\epsilon)\Pr[\calx],(1+\epsilon)\Pr[\calx]]\enspace ,
\]
as desired.
\end{proof}

We will use the following lemmas from previous works. 

\begin{lemma}[\cite{Ewens04}]\label{lem:neutral_fix_pr}
The solution of the fixation problem for $t_1$ is at least $\frac{1}{n}$ and the solution of the extinction problem for $t_1$ is at least $1-\frac{1}{n}$.
\end{lemma}

\begin{lemma}[\cite{Diaz14}]\label{lem:extinct_lowerbound}
The solution of the extinction problem is greater than $\frac{1}{n+r}$ for $t_2$.
\end{lemma}

We next show the following lemma, which will be useful in the correctness proofs of our algorithms.

\begin{lemma}\label{lem:u_zr}
Consider running MetaSimulation($t$,$z$,$u$,$\distr$) on an instance where $w(t_1)=r$, 
for some distribution $\distr$, type $t$ and numbers $z$, $0<\epsilon<\frac{1}{2}$, when $u=u_{z,r}$. 
Then, $1-\Pr[\cale_u]\leq \min(z^{-5},\epsilon^{5})$ and the probability that MetaSimulation does 
not output a number is smaller than $\min(z^{-4},\epsilon^{4})$. 
The run time $T$ is 
\[
T=\begin{cases}
O(n\Delta z \cdot \max(\log z,\log \epsilon^{-1})) &\ift r\geq 2\Delta\\
O(n\Delta^2\frac{z}{\min(r-1,1)} \cdot \max(\log z,\log \epsilon^{-1})) &\ift 1+\frac{1}{n\cdot \Delta}\leq r< 2\Delta\\
O(n^2\Delta^3z \cdot \max(\log z,\log \epsilon^{-1})) &\ift r<1+\frac{1}{n\cdot \Delta}
\end{cases}
\]
\end{lemma}

\begin{proof}
Let $t,z,\distr,r,\epsilon$ be as in the lemma statement. Also, let $u=u_{z,r}$.
The run time of MetaSimulation($t$,$z$,$u$,$\distr$) is $O(uz\Delta)$ (by Theorem~\ref{thm:sampling}), which leads to the desired run time when inserting $u_{z,r}$.
By choice of $u$, we have that \[
1-\Pr[\cale_u]\leq 2^{-5\max(\log z,\log \epsilon^{-1})}=\min(z^{-5},\epsilon^{5})\enspace ,
\]
using Theorem~\ref{thm:fixationtime}. Thus, any single simulation takes more than $u$ time with probability less than  $\min(z^{-5},\epsilon^{-5})$. 
Union bounding the $z$ simulations, we see that the probability that any one takes more than $u$ time is at most $\min(z^{-4},\epsilon^{-4})$. 
Note that only in case one simulation takes more than $u$ time does the algorithm output something which is not a number.
\end{proof}

The remainder of this section is about showing the correctness and running times of the various algorithms.

\begin{lemma}\label{lem:algo1}
Algorithm $\algo1$ is correct and runs in time $T_1$, for  
\[
T_1=\begin{cases}
O(n^2\cdot \Delta\cdot \epsilon^{-2} \cdot \max(\log n,\log \epsilon^{-1})) &\ift r\geq 2\Delta\\
O(\frac{n^2\cdot \Delta^2}{\epsilon^2\cdot \min(r-1,1)} \cdot \max(\log n,\log \epsilon^{-1})) &\ift 1+\frac{1}{n\cdot \Delta}\leq r< 2\Delta\\
O(n^3\cdot \Delta^3\cdot \epsilon^{-2} \cdot \max(\log n,\log \epsilon^{-1})) &\ift 1<r<1+\frac{1}{n\cdot \Delta}
\end{cases}
\]
\end{lemma}
\begin{proof}
Recall if $r=1$, then the fixation probability of $t_1$ is $\frac{1}{n}$. 
We consider the other cases.
Otherwise, the output of $\algo1$ is MetaSimulation($t_1$,$z$,$u_{z,r}$,$\distr$), 
for $z=48\cdot \frac{n}{\epsilon^2}$. The run time then follows from Lemma~\ref{lem:u_zr}.

We will utilize Lemma~\ref{lem:meta_condition} to show the correctness of $\algo1$.
We thus need to ensure that $\Pr[\cale_u]\geq 1-\frac{\epsilon\cdot \Pr[\calf_{t_1}]}{4}$ and that with probability 
$\frac{3}{4}$ our algorithm outputs $x$, for \[x\in [(1-\epsilon/2)\Pr[\calf_{t_1}\mid \cale_u],(1+\epsilon/2)\Pr[\calf_{t_1}\mid \cale_u]]\enspace .\]
According to Lemma~\ref{lem:u_zr}, we have that \[\Pr[\cale_u]\geq 1-\min(z^{-5},\epsilon^{5})=1-z^{-5}\geq 1-\frac{\epsilon}{4n}\geq 1-\frac{\epsilon\cdot \Pr[\calf_{t_1}]}{4}\enspace ,\] using that $\Pr[\calf_{t_1}]$ is at least $\frac{1}{n}$ according to Lemma~\ref{lem:neutral_fix_pr} as desired.

\subparagraph*{Error sources}
We see that there are two cases where $\algo1$ does not output a number in $[(1-\epsilon)\cdot \Pr[\calf_{t_1}],(1+\epsilon)\cdot \Pr[\calf_{t_1}]]$: 
\begin{enumerate}
\item The algorithm does not output a number. This happens with probability at most \[\min(z^{-4},\epsilon^{4})=(48\cdot \frac{n}{\epsilon^2})^{-4}<(48\cdot \frac{2}{(1/2)^2})^{-4}=384^{-4}\enspace ,\] according to Lemma~\ref{lem:u_zr}.
\item The algorithm outputs a number outside $[(1-\epsilon)\cdot \Pr[\calf_{t_1}],(1+\epsilon)\cdot \Pr[\calf_{t_1}]]$.
\end{enumerate}

\subparagraph*{Simulations are far from expectation}
 Let the output of the algorithm, conditioned on it being a number, be the random variable $\calb$. 
 Observe that if 
 \[
 \calb\in [(1-\epsilon/2)\cdot \Pr[\calf_{t_1}\mid \cale_u],(1+\epsilon/2)\cdot \Pr[\calf_{t_1}\mid \cale_u]] \enspace ,
 \] 
 then, according to Lemma~\ref{lem:meta_condition}, we have that $\calb\in [(1-\epsilon)\cdot \Pr[\calf_{t_1}],(1+\epsilon)\cdot \Pr[\calf_{t_1}]]$, since we already argued that $\Pr[\cale_u]\geq 1-\frac{\epsilon\cdot \Pr[\calf_{t_1}]}{4}$.
 We are thus interested in the probability 
 \[
 \Pr[|\calb-\Pr[\calf_{t_1}\mid \cale_u]|\geq \epsilon/2\Pr[\calf_{t_1}\mid \cale_u]] \enspace .
 \]  
Applying Lemma~\ref{lem:meta_condition}, with $\calx=\calf_{t_1}$, $\cale=\cale_u$ and $x=\Pr[\calf_{t_1}\mid \cale_u]$, 
we see that $\Pr[\calf_{t_1}\mid \cale_u]\geq (1-\epsilon)\cdot\Pr[\calf_{t_1}]\geq \frac{\Pr[\calf_{t_1}]}{2}\geq \frac{1}{2n}$. 
Thus, 
\[
\Pr[|\calb-\Pr[\calf_{t_1}\mid \cale_u]|\geq \epsilon/2\Pr[\calf_{t_1}\mid \cale_u]]\leq \Pr[|\calb-\Pr[\calf_{t_1}\mid \cale_u]|\geq \frac{\epsilon}{4n}] \enspace .
\] 
Since $\calb$ is the average of $z$ independent simulations of a binary random variable, each of which are~1 with probability 
$\Pr[\calf_{t_1}\mid \cale_u]\geq \frac{1}{n}$ (by Lemma~\ref{lem:neutral_fix_pr}), we can apply the multiplicative Chernoff bound and obtain:
\[
\Pr[|z\calb-z\Pr[\calf_{t_1}\mid \cale_u]|\geq \frac{z\epsilon}{4n}]\leq 
2e^{-2z\frac{\epsilon^2}{3\cdot 4^2n}}=2e^{-2}\enspace .
\]
We have that $384^{-4}+2e^{-2}<1/3$. 
The desired result follows.
\end{proof}

\begin{lemma}\label{lem:algo2}
Algorithm $\algo2$ is correct and runs in time $T_2$, for  \[
T_2=\begin{cases}
O(n\cdot \Delta\cdot \epsilon^{-2}\cdot\log \epsilon^{-1}) &\ift r\geq 2\Delta\\
O(\frac{n\cdot \Delta^2\cdot\log \epsilon^{-1}}{\epsilon^2\cdot \min(r-1,1)}) &\ift 1+\frac{1}{n\cdot \Delta}\leq r< 2\Delta\\
O(n^2\cdot \Delta^3\cdot \epsilon^{-2}\cdot\log \epsilon^{-1}) &\ift 1<r<1+\frac{1}{n\cdot \Delta}
\end{cases}
\]
\end{lemma}

\begin{proof}
The proof is similar to Lemma~\ref{lem:algo1}, except applying multiplicative Chernoff Bound to 
\[\Pr[|z\calb-z\Pr[\calf_{t_1}\mid \cale_u]|\geq z\epsilon\Pr[\calf_{t_1}\mid \cale_u]/2]\leq 
\Pr[|\calb-\Pr[\calf_{t_1}\mid \cale_u]|\geq \frac{\epsilon}{4}]\enspace ,\] 
using that $\Pr[\calf_{t_1}\mid \cale_u]\geq 1-\frac{1}{n}\geq \frac{1}{2}$ 
by Lemma~\ref{lem:neutral_fix_pr}.
\end{proof}

\begin{lemma}\label{lem:algo3}
Algorithm $\algo3$ is correct and runs in time $T_3$, for  \[
T_3=\begin{cases}
O(n\cdot \Delta\cdot \epsilon^{-2}\cdot \log \epsilon^{-1}) &\ift r\geq 2\Delta\\
O(\frac{n\cdot \Delta^2\cdot \log \epsilon^{-1}}{\epsilon^2\cdot \min(r-1,1)}) &\ift 1+\frac{1}{n\cdot \Delta}\leq r< 2\Delta\\
O(n^2\cdot \Delta^3\cdot \epsilon^{-2} \cdot \log \epsilon^{-1}) &\ift 1\leq r<1+\frac{1}{n\cdot \Delta}
\end{cases}
\]
\end{lemma}
\begin{proof}
The proof is somewhat similar to Lemma~\ref{lem:algo1}, but there are some differences and thus we include the proof.

The output of $\algo3$ is MetaSimulation($t$,$z$,$u_{z,r}$,$\distr$), for $z=6/\epsilon^{2}$ and $\distr$ the distribution assigning probability~1 to $\f$. 
We see that \begin{align*}
\Pr[\calf_t]&=\Pr[\cale_u]\cdot \Pr[\calf_t\mid \cale_u]+(1-\Pr[\cale_u])\cdot \Pr[\calf_t\mid \neg \cale_u]\geq \Pr[\calf_t\mid \cale_u]\cdot \Pr[\calf_t\mid \cale_u]\\
&\geq (1-z^{-5})\cdot \Pr[\calf_t\mid \cale_u]\geq \Pr[\calf_t\mid \cale_u] - z^{-5}\geq \Pr[\calf_t\mid \cale_u] - \epsilon/2\enspace ,
\end{align*}
using Lemma~\ref{lem:u_zr}. 
Also, \[\Pr[\calf_t]=\Pr[\cale_u]\cdot \Pr[\calf_t\mid \cale_u]+(1-\Pr[\cale_u])\cdot \Pr[\calf_t\mid \neg \cale_u]\leq \Pr[\calf_t\mid \cale_u]+z^{-5}\leq \Pr[\calf_t\mid \cale_u]+\epsilon/2 \enspace .\]
Thus, we just need to ensure that we output a number $x$ in 
$[\Pr[\calf_t\mid \cale_u]-\epsilon/2,\Pr[\calf_t\mid \cale_u]+\epsilon/2]$ with probability at least $\frac{3}{4}$.

We output a number with probability at least $1-z^{-4}>1-24^{-4}$ according to Lemma~\ref{lem:u_zr}.
In case we output a number, we can apply Hoeffding's inequality, since the output is the average of $z$ trials, and obtain:\[
\Pr[|\calb-\Pr[\calf_{t}\mid \cale_u]|\geq \frac{\epsilon}{2}]\leq 2e^{-2z\frac{\epsilon^2}{2^2}}=2e^{-3}\enspace .
\]
We see that we thus fail to output a number in $[\Pr[\calf_t]-\epsilon,\Pr[\calf_t]+\epsilon]$ with probability at most $2e^{-3}+24^{-4}<1/10$.
This completes the proof.
\end{proof}

Theorem~\ref{thm:FPRAS} follows from Lemmas~\ref{lem:algo1}, \ref{lem:algo2}, \ref{lem:algo3}.

\section{Algorithm for extinction of $t_2$}
In this section we present the algorithm for extinction of $t_2$.
First we present an approximation of the fixation probability.

\subsection{Approximating the fixation probability}
We will next show that whenever $r$ is sufficiently big, $\frac{1}{r}$ is a good approximation of the solution of the extinction problem for $t_2$.
To do so, we first argue that with probability at most $\frac{1}{r+1}$ does $t_1$ go extinct in the first effective step, when there is only a single random node which is a member of $t_1$.

\begin{lemma}\label{lem:extinct_in_1}
The probability that $t_1$ goes extinct in the first effective step is at most $\frac{1}{r+1}$, when initially a single random node is of type $t_1$.
\end{lemma}
\begin{proof}
The proof is similar to the proof of \cite[Lemma~4]{Diaz14}, except that instead of lower bounding, we try to upper bound.
For each node $v\in V$, let $Q(v)=\sum_{(u,v)\in E}\frac{1}{\deg u}$. Observe that $\sum_{v\in V}Q(v)=n$.
Let $p$ be the probability that $t_1$ does not go extinct in the first effective step, when initially a single random node is of type $t_1$.
From the proof of \cite[Lemma~4]{Diaz14}, we see that $p=\frac{r}{n}\sum_{v\in V}\frac{1}{r+Q(v)}$. 
We want to find the minimum $p$ such that $p=\frac{r}{n}\sum_{i=1}^n\frac{1}{r+q_i}$, subject to the constraint that $q_i>0$ and $\sum_{i=1}^nq_i=n$.

\begin{claim}
The number $p$ is minimized for $q_i=1$ for all $i$.
\end{claim}
\begin{proof}
We will argue that $p$ is minimized for $q_i=1$ for all $i$, as follows:
Our argument is by contradiction. For a vector $q$ of length $n$, let $p^q=\frac{r}{n}\sum_{i=1}^n\frac{1}{r+q_i}$.
Assume that $q^*=(q_i^*)_{i\in [1,\dots,n]}$ is a vector of $q_i$'s that minimizes $p^q$ and that $q_j^*\neq q_k^*$ for some $j,k$ (this is the case if not all are~1). Let $2z=q_j^*+q_k^*$. Consider the vector $q'$ such that $q_i'=q_i^*$ for all $i$, except that $q_j'=q_k'=z$. We will argue that $p^{q'}<p^{q^*}$, contradicting that $q^*$ minimizes $p^q$.
We have that 
\begin{align*}
\frac{1}{r+q_j^*}+\frac{1}{r+q_k^*}&=\frac{r+q_k^*+r+q_j^*}{(r+q_j^*)(r+q_k^*)}=\frac{2r+2z}{r^2+2rz+q_j^*q_k^*}>\frac{2r+2z}{r^2+2rz+z^2}=\frac{1}{r+z}+\frac{1}{r+z}\enspace ,
\end{align*}
since, $q_j^*q_k^*<z^2$ and thus, since all other terms of $p^{q'}$ is equal to $p^{q^*}$, we have that $p^{q'}<p^{q^*}$. Thus, in the worst case $q_i=1$ for all $i$.
\end{proof}

Hence, $p\geq \frac{r}{r+1}=1-\frac{1}{r+1}$. This completes the proof.
\end{proof}

\begin{reptheorem}{thm:too_big_r}
Let $0<\epsilon<\frac{1}{2}$ be given.
If \[r\geq \max(\Delta^2,n)/\epsilon\enspace ,\]
then \[\frac{1}{r}\in [(1-\epsilon)\cdot \rho,(1+\epsilon)\cdot \rho]\enspace ,\] where $\rho$ is the solution of the extinction problem for $t_2$.
\end{reptheorem}

\begin{proof}
Fix $0<\epsilon<\frac{1}{2}$ and $r\geq \max(\Delta^2,n)/\epsilon$. Let $\rho$ be the solution of the extinction problem for $t_2$.
We consider two cases.

\subparagraph*{$\bm{\frac{1}{r}}$ is below $\bm{(1+\epsilon)\cdot \rho}$}
By \cite[Lemma 4]{Diaz14} (recalled as Lemma~\ref{lem:extinct_lowerbound}), we have that $\rho\geq \frac{1}{n+r}$.
Thus, $(1+\epsilon)\cdot \rho\geq \frac{1+\epsilon}{n+r}> \frac{1+\epsilon}{(1+\epsilon)r}=\frac{1}{r}$, as desired, since $n/\epsilon\leq r$.

\subparagraph*{$\bm{\frac{1}{r}}$ is above $\bm{(1-\epsilon)\cdot \rho}$}
The probability that after the first effective step there are two nodes of type $t_1$ is at least $1-\frac{1}{r+1}$, by Lemma~\ref{lem:extinct_in_1}.
In every subsequent step before fixation, we have probability at most $\frac{1}{r/\Delta+1}$, by Lemma~\ref{lem:large_r}, of decreasing the number of nodes of type $t_1$ by~1 and otherwise we increase the number by~1.

\subparagraph*{Modeling as Markov chain}
Consider a biased random walk on a line with absorbing boundaries (this problem is also sometimes called Gambler's Ruin). More precisely, let $M$ be the following Markov chain, with states $\{0,\dots,n\}$, and whenever $0$ or $n$ is reached, they will never be left (i.e., the states are absorbing). Also, in each other state $i\in \{1,\dots,n-1\}$, there is a probability of $p=\frac{1}{r/\Delta+1}$ of going to $i-1$ and otherwise, the next state is $i+1$.
 The fixation probability for $t_2$, if fixation did not happen in the first step, is at most the absorption probability in state~0 in $M$ starting in state 2. 
The absorption probability in state~0 of $M$, starting in state 2, is well-known, and it is 
\[
\frac{(\frac{1-p}{p})^{n-2}-1}{(\frac{1-p}{p})^{n}-1}=\frac{(\frac{r}{\Delta})^{n-2}-1}{(\frac{r}{\Delta})^{n}-1}<\frac{(\frac{r}{\Delta})^{n-2}}{(\frac{r}{\Delta})^{n}}= \Delta^2/r^2\leq \epsilon/r
\]
using that $\frac{1-p}{p}=\frac{r}{\Delta}$ in the first equality, $\frac{a-1}{b-1}<\frac{a}{b}$ for all $1<a<b$ in the first inequality and $\Delta^2/\epsilon\leq r$ in the last inequality. 
Hence, the probability to fixate for $t_2$ is below $\epsilon/r$ if $t_2$ does not fixate in the first step.
Thus, the probability $\rho$ of fixation for $t_2$ is at most  $(1+\epsilon)\cdot \frac{1}{r}$.
Observe that 
\[
(1-\epsilon)\rho\leq (1-\epsilon)(1+\epsilon)\cdot \frac{1}{r}<\frac{1}{r}, \quad \text{ since } (1-\epsilon)(1+\epsilon)=1-\epsilon^2<1 \enspace .
\]
Hence, $\frac{1}{r}$ is in $[(1-\epsilon)\cdot \rho,(1+\epsilon)\cdot \rho]$ as desired.
\end{proof}

\subsection{Algorithm}

\subparagraph*{Algorithm $\bm{\algo4}$}
We consider the extinction problem for $t_2$.
Algorithm $\algo4$ is as follows:
\begin{enumerate}
\item If $r\geq \max(\Delta^2,n)/\epsilon$ return $\frac{1}{r}$.
\item Let $\distr$ be the uniform distribution over the $n$ type functions where exactly one node is $t_1$.
\item 
Return MetaSimulation($t_2$,$z$,$u_{z,r}$,$\distr$), for $z=24\frac{(n+r)^2}{\epsilon^2}$.
\end{enumerate}

We next prove the following theorem.
\begin{theorem}\label{thm:FPRAS2}
Let $G$ be a connected undirected graph of $n$ nodes with highest degree $\Delta$, 
divided into two types of nodes $t_1,t_2$, such that $r=w(t_1) > w(t_2)=1$. 
Let $\frac{1}{2}>\epsilon>0$.

\begin{itemize}
\item {\bf Extinction problem given $\bm{t_2}$}
Let $w=\max(\Delta^2,n)/\epsilon$ and let $T$ be \[
T=\begin{cases}
O(1)&\ift r\geq w\\
O(r^2\cdot n\cdot \Delta\cdot \epsilon^{-2}\cdot \max(\log r,\log \epsilon^{-1})) &\ift  n\leq r< w\\
O(n^3\cdot \Delta\cdot \epsilon^{-2} \cdot \max(\log n,\log \epsilon^{-1})) &\ift 2\Delta\leq r< n\\
O(\frac{n^3\cdot \Delta^2}{\epsilon^2\cdot \min(r-1,1)} \cdot \max(\log n,\log \epsilon^{-1})) &\ift 1+\frac{1}{n\cdot \Delta}\leq r< 2\Delta\\
O(n^4\cdot \Delta^3\cdot \epsilon^{-2}\cdot \max(\log n,\log \epsilon^{-1})) &\ift 1<r<1+\frac{1}{n\cdot \Delta} \enspace .
\end{cases}
\]
Algorithm $\algo4$ is an FPRAS algorithm, with running time $T$, that with probability atleast $\frac{3}{4}$ 
outputs a number in $[(1-\epsilon)\cdot \rho,(1+\epsilon)\cdot \rho]$, where $\rho$ 
is the solution of the extinction problem given $t_2$.
\end{itemize}
\end{theorem}

\begin{remark}\label{rem:appendix}
A detailed comparison of our results and previous results of~\cite{Diaz14} is presented in Table~\ref{tab:appendix}.
Note that we do not present approximation of the fixation problem for $t_2$ (i.e., fixation for the type with the smallest fitness). 
This is because in this case the fixation probabilities can vary greatly, while being exponentially small, even for constant $r$: It is close to $\frac{1-r^{-2}}{1-r^{-2n}}$ for star graphs~\cite{Nowak05} and precisely $\frac{1-r^{-1}}{1-r^{-n}}$ for complete graphs~\cite{Nowak05}. For, for instance, $r=1/2$ it is thus close to \[\frac{1-2^2}{1-2^{2n}}=\frac{2^2-1}{2^{2n}-1}\approx 2^{-2n+2}\] for star graphs and  \[\frac{1-2}{1-2^{n}}=\frac{2-1}{2^{n}-1}=\frac{1}{2^{n}-1}\] for complete graphs. Observe that those two numbers differs by an exponentially large factor.
For this case~\cite{Diaz14} also do not present any polynomial-time algorithm.
\end{remark}

\begin{table}
\center
\begin{tabular}{| l l| c| c| }
\hline
   & & All steps & Effective steps \\\hline
    \#steps in expectation& $r>1$ & $O(n^2\Delta^2)$ & $O(n\Delta)$ \\\hline
& $r=1$ & $O(n^3\Delta^3)$ & $O(n^2\Delta^2)$ \\\hline
& $r<1$ & $O(n\Delta^2)$ & $O(\Delta)$ \\\hline
    Concentration bounds& $r>1$ & $\Pr[\tau \geq \frac{n^2\Delta^2rx}{r-1}]\leq 1/x$ & $\Pr[\tau \geq \frac{6n\Delta x}{\min(r-1,1)}]\leq 2^{-x}$ \\
    \hline
& $r=1$ & $\Pr[\tau \geq n^3\Delta^3x]\leq 1/x$ & $\Pr[\tau \geq 4n^2\Delta^2x]\leq 2^{-x}$ \\\hline
 & $r<1$ & $\Pr[\tau \geq \frac{n\Delta^2x}{1-r}]\leq 1/x$ & $\Pr[\tau \geq \frac{6n\Delta x}{1-r}]\leq 2^{-x}$ \\\hline
\multicolumn{2}{|l|}{Sampling a step}  & $O(1)$ & $O(\Delta)$ \\\hline
\multicolumn{2}{|l|}{Fixation for $t_1$}  & $O(n^6 \Delta^2\epsilon^{-4})$ & $O(n^2 \Delta^2\epsilon^{-2}(\log n+\log\epsilon^{-1}))$ \\\hline
\multicolumn{2}{|l|}{Extinction for $t_1$}  & $O(n^6 \Delta^2\epsilon^{-4})$ & $O(n \Delta^2\epsilon^{-2}\log\epsilon^{-1})$ \\\hline
\multicolumn{2}{|l|}{Extinction for $t_2$} & $O(n^5 \Delta^2\epsilon^{-4})$ & $O(n^3\Delta^2\epsilon^{-2}(\log n+\log\epsilon^{-1}))$ \\\hline
\multicolumn{2}{|l|}{Generalized fixation (additive)} & $-$ & $O(n \Delta^2\epsilon^{-2}\log\epsilon^{-1})$\\\hline
\end{tabular}
\caption{Comparison with previous work, for constant $r$. 
We denote by $n$, $\Delta$, $\tau$, and $\epsilon$, the number of nodes, the maximum degree, the random variable for the fixation time, and
the approximation factor, respectively.
The results in the column ``All steps'' is from \cite{Diaz14}, except that we present their dependency on $\Delta$ which was considered as $n$. 
The entry marked $-$ is done so, since \cite{Diaz14} did not provide an algorithm for the generalized fixation problem. The results of the column ``Effective steps'' is the results of this paper\label{tab:appendix}}
\end{table}

\subparagraph*{Variant of Theorem~\ref{thm:too_big_r} for directed graphs}
A variant of Theorem~\ref{thm:too_big_r} can also be proven for strongly connected (i.e. graphs where there is a path from any node to any node) directed graphs. However, the variant requires $O(m)$ computation.

In essence, the proof of Theorem~\ref{thm:too_big_r} is as follows: The probability that $t_2$ fixates in 1 step is always nearly $\frac{1}{r}$. 
If $t_2$ does not fixate in 1 step (and there are thus at least 2 nodes of type $t_1$), then $t_2$ has probability at most $\epsilon/r$ of fixating.

This suggests a simple scheme for the directed case: (1)~Find the probability $\rho_1$ that $t_2$ fixates in 1 step. 
(2)~Show that fixating if there are two states of type $t_2$ is small compared to $\rho_1$.

It is easy to find the probability that $t_2$ fixates in 1 step in any (even directed) graph, by simply consider each start configuration and computing the probability that fixation happens in 1 step. For a fixed configuration $\f$ with one node $v$ of type $t_1$, this takes time equal to the in- and out-degree of $v$ (assuming that the out- and in-degree of all nodes are known). Overall this thus takes $O(m)$ operations. The probability of fixation in 1 step for a configuration is in $[\frac{1}{\Delta r+1},\frac{n-1}{r+(n-1)}]$, recalling that $\Delta$ is the max out-degree). (This bound is easy to establish: 
Let $v$ be the node of type $t_1$. The bounds follow from that at least one node $u$ of type $t_2$ is such that $v$ is a successor of $u$ (resp. $u$ is a successor of $v$), because the graph is connected. On the other hand, all nodes of type $t_2$ might be a successor of $v$ (or the other way around))
It seems reasonably to consider that operations on that small numbers can be done in constant time. Therefore, we get a $O(m)$ running time.

To then show that fixating for $t_2$ is small (compared to $\rho_1$) one needs a variant of Lemma~\ref{lem:large_r} for directed graphs. This can easily be done, but just shows that the probability to increase the number of members of $t_2$ is at most $p:=\frac{1}{r/(n\Delta)+1}$ (i.e. it is a factor of $n$ worse as compared to undirected graphs). 
The remainder of the proof is then similar (i.e. we consider the Markov chain $M$ on $0,1,\dots,n$, starting in state 2, that has pr. $p$ of decrementing and otherwise increments). 
We get that the pr. of fixating for $t_2$ if it did not do it in the first step is at most $\epsilon/(2\Delta r)<\epsilon/(\Delta r+1)$ if $r\geq 2\Delta^3 n^2/\epsilon$. Hence, $\rho_1\in [(1-\epsilon)\cdot \rho,(1+\epsilon)\cdot \rho]$ for such $r$.
This leads to the following theorem.

\begin{theorem}
Let $\epsilon>0$ be given.
Let $G$ be a directed graph and let $r\geq 2\Delta^3 n^2/\epsilon$. 

Consider the extinction problem for $t_2$. Let $\rho$ be the fixation pr. and let $\rho_1$ be the pr. of fixation in 1~effective step. Then \[\rho_1\in [(1-\epsilon)\cdot \rho,(1+\epsilon)\cdot \rho]\enspace .\]
\end{theorem}

\end{document}